\documentclass{article} % For LaTeX2e
\usepackage{iclr2026_conference,times}

\usepackage{amsmath,amsfonts,bm}

\def\eqref#1{equation~\ref{#1}}
\def\1{\bm{1}}

\DeclareMathAlphabet{\mathsfit}{\encodingdefault}{\sfdefault}{m}{sl}
\SetMathAlphabet{\mathsfit}{bold}{\encodingdefault}{\sfdefault}{bx}{n}

\newcommand{\R}{\mathbb{R}}

\usepackage{hyperref}
\usepackage{url}
\usepackage{xcolor}   
\usepackage[acronym]{glossaries}
\usepackage{graphicx}
\usepackage{wrapfig}
\usepackage{makecell}
\usepackage{natbib}
\usepackage{amsmath}
\usepackage{amsmath, amssymb, amsthm}
\usepackage[capitalize]{cleveref}
\usepackage{subcaption}
\usepackage{booktabs}
\usepackage{adjustbox}
\usepackage{threeparttable}
\usepackage{amsfonts} 
\usepackage{multirow}
\usepackage[skip=2pt]{caption}
\newcommand{\method}[2]{#1{\tiny$_\text{#2}$}}

\newtheorem{theorem}{Theorem}[section] % numbers within sections
\newtheorem{proposition}[theorem]{Proposition}

\newtheorem{corollary}[theorem]{Corollary}

\newacronym{cfd}{CFD}{computational fluid dynamics}
\newacronym{ddpm}{DDPM}{denoising diffusion probabilistic model}
\newacronym{mlp}{MLP}{multilayer perceptron}
\newacronym{cnn}{CNN}{convolutional neural network}
\newacronym{gnn}{GNN}{graph neural network}
\newacronym{no}{NO}{Neural Operator}
\newacronym{fno}{FNO}{Fourier Neural Operator}
\newacronym{fft}{FFT}{Fast Fourier Transform}
\newacronym{ifft}{iFFT}{inverse fast Fourier transform}
\newacronym{pde}{PDE}{Partial Differential Equation}
\newacronym{pinn}{PINN}{physics-informed neural network}
\newacronym{ode}{ODE}{Ordinary Differential Equation}
\newacronym{ae}{AE}{autoencoder}
\newacronym{gan}{GAN}{generative adversarial model}
\newacronym{cnf}{CNF}{Continuous Normalizing Flow}
\newacronym{slp}{SLP}{stochastic linear path}
\newacronym{affine}{Affine-OT}{affine optimal transport}
\newacronym{river}{RIVER}{Random frame conditioned flow Integration for VidEo pRediction}
\newacronym{ve}{VE-diff}{variance exploding diffusion}
\newacronym{vp}{VP-diff}{variance preserving diffusion}
\newacronym{tempo}{TempO}{Temporal Operator flow matching}
\newacronym{vit}{ViT}{Vision Transformer}
\newacronym{swe}{SWE}{shallow water equation}
\newacronym{rd}{RD-2D}{2D reaction diffusion}
\newacronym{nsv}{NS-$\omega$}{2D incompressible Navier-Stokes vorticity}
\newacronym{ns2c}{NS-$v$}{2D incompressible Navier-Stokes velocity}
\newacronym{nfe}{NFE}{number of function evaluations}
\makeglossaries

\title{Operator Flow Matching for Timeseries Forecasting}

\author{Yolanne Yi Ran Lee \\
Department of Computer Science\\
University College London\\
Gower Street, London WC1E 6BT, UK \\
\texttt{yolanne.lee.19@ucl.ac.uk} \\
\And
Kyriakos Flouris \\
MRC Biostatistics Unit \\
University of Cambridge \\
Forvie, Robinson Way, Cambridge CB2 0SR, UK \\
\texttt{kf318@cam.ac.uk}
}

\iclrfinalcopy % Uncomment for camera-ready version, but NOT for submission.
\begin{document}

% \begin{enumerate}
%     \item IMP Rewrite the abstract to the narative of the paper - maybe be short and very on the nose can be a good idea
%     \item Can sharpen language of introduction to safe space - also maybe remove some sentences that are not needed - paragraph 3 maybe can be shorter
%     \item shorten Relevant work byt removing references that have alaredy be made in the intro
%     \item IMP revisit method section - follow intro line of thought and enumerate (i) (ii) components and tricks you did?
%     \item NEW and IMP at the end of each RW paragraphs with a short statement specify how our method is improving/different/scope from them
%     \item there are some styling issues where tables are next to graphs that do not match size and text wrapping tha looks werid.
%     \item tables should have short titles i think
%     \item figures check consistent titles and style of titles
%     \item consistent notation in language and equations - kf i have checked the methos more or less
%     \item Check pacing overall, some parts like ablations maybe too long and some parts like spectral analysis are missing explanations - eg. inset
%     \item check that there are no paragraphs (enter) at random places that are not supporsed to be there
%     \item Are sections consistent and correct?
%     \item missing parts in description of forecasting 40 ts
%     \item figure 1 can be made more tight
%     \item revisit important parts Absttact conclusions - last paragraph of intro at the end before submitting
% \end{enumerate}

% \newpage
\maketitle

\begin{abstract}
Forecasting high-dimensional, PDE-governed dynamics remains a core challenge for generative modeling. Existing autoregressive and diffusion-based approaches often suffer cumulative errors and discretisation artifacts that limit long, physically consistent forecasts. Flow matching offers a natural alternative, enabling efficient, deterministic sampling. We prove an upper bound on FNO approximation error and propose TempO, a latent flow matching model leveraging sparse conditioning with channel folding to efficiently process 3D spatiotemporal fields using time-conditioned Fourier layers to capture multi-scale modes with high fidelity. TempO outperforms state-of-the-art baselines across three benchmark PDE datasets, and spectral analysis further demonstrates superior recovery of multi-scale dynamics, while efficiency studies highlight its parameter- and memory-light design compared to attention-based or convolutional regressors.
% Flow matching has recently emerged as a competitive class of generative models, matching or surpassing diffusion approaches in sample quality while offering stronger theoretical guarantees and dramatically faster inference. However, challenges remain in high dimensional data due to the computational bottleneck of numerically solving an ODE during the sampling process, as well as in temporal conditioning for trajectory rollouts. We propose Operator Flow Matching, which builds on recent work in developing spatio-temporal video flow matching models as well as contemporary physics-informed models to produce a temporal forecasting model for scientific data. In particular, we design a latent space flow matching model which leverages the Fourier Neural Operator (FNO) as the vector field regressor. The proposed model exhibits state of the art performance using extended metrics to capture additional performance measurements across high dimensional spatio-temporal partial differential equation benchmark datasets compared to comparable flow matching video generation models. Additionally, it improves training stability and allows for accelerated sampling, resulting in faster training and test time generation.
\end{abstract}

\section{Introduction}

Generative artificial intelligence has brought unparalleled creative and scientific potential, with models capable of producing images~\citep{hatamizadehDiffiTDiffusionVision2025}, video~\citep{bar-talLumiereSpaceTimeDiffusion2024}, audio~\citep{juNaturalSpeech3ZeroShot2024}, and text~\citep{grattafioriLlama3Herd2024} that rival human quality. From autoregressive transformers to diffusion models and energy-based approaches, the landscape of generative AI is rich and diverse, offering multiple pathways to model complex data distributions. At the core of this revolution are probabilistic generative models, which learn to sample from complex, high-dimensional distributions~\citep{flourisCanonicalNormalizingFlows2023}. Among these, flow matching models have emerged as a class of generative models which learn to transform a simple prior distribution to a more complex data distribution as a continuous transformation. This direct, simulation-free approach enables both efficiency and precision, offering a new lens on modeling complex systems~\citep{lipmanFlowMatchingGenerative2023}.

Despite recent advances, forecasting high-dimensional temporal dynamics remains challenging~\citep{leeAutoregressiveRenaissanceNeural2023a}. Deep learning models are computationally expensive and often fail catastrophically after a few dozen timesteps due to compounding errors in autoregressive predictions~\citep{ansariChronosLearningLanguage2024}. Even with the advent of large language models and their remarkable ability to generate, models that attempt to leverage them for forecasting face limitations of discretisation and tokenisation~\citep{ansariChronosLearningLanguage2024}, offering little practical benefit relative to their computational cost~\citep{tanAreLanguageModels2024}. Modern generative models have been proven capable of generating visually compelling and coherent videos~\citep{johnSurveyProbabilisticGenerative2024}, but critically lack the fine-grained control required to be used in scientific and engineering contexts.

Recent foundation models for forecasting include GenCast for weather~\citep{priceProbabilisticWeatherForecasting2025} and Chronos for general time series~\citep{ansariChronosLearningLanguage2024}, demonstrate the promise of large-scale pretraining. These models leverage massive datasets across multiple domains resulting in strong zero- and few-shot transferability. Chronos captures coarse, long-range correlations remarkably long timespans; however, the granularity, i.e. prediction length still falls at an average of 22 across 55 datasets, with only 7 tasks exceeding 30 steps~\citep{ansariChronosLearningLanguage2024}. GenCast, likewise, can generate 15-day global weather forecasts, but at a granularity of 12 hours, around 30 steps. True progress requires models capable of deterministic yet flexible generation, able to explore plausible trajectories while respecting physical constraints to then select precise forecasts out of the space of plausible predictions~\citep{guoDynamicalDiffusionLearning2025}. Although the short to mid term range is a popular horizon to explore~\citep{limElucidatingDesignChoice2025}, the goal is to generate long-horizon predictions on the order of 30 timesteps or more, generating trajectories that are not just plausible, but physically consistent.

Fundamentally, models relying on discretisation or tokenisation are not ideal for continuous, \gls{pde}-governed dynamics. Demonstrating smooth trajectories in state space which generalise to long forecasting horizons would show greater fidelity to the underlying physics~\citep{leeInterpretableNeuralPDE2023a}. Other existing efforts which leverage diffusion~\citep{molinaroGenerativeAIFast2025, yaoGuidedDiffusionSampling2025,huangDiffusionPDEGenerativePDESolving2024} move toward more natural representations, but are themselves fundamentally tied to stochastic dynamics. Instead, a natural choice for such a problem is flow matching, where the vector field regression is closer to learning \gls{pde} operators which are themselves vector fields describing time-derivatives, and learn deterministic dynamics with potentially more efficient \gls{ode}-based sampling in contrast to the denoising process of diffusion models. Existing flow matching methods have individually worked toward video generation~\citep{davtyanEfficientVideoPrediction2023,jinPyramidalFlowMatching2025} and \gls{pde} single-step prediction~\citep{kerriganFunctionalFlowMatching2023}, but thus far have not been thoroughly tested for long-horizon temporal forecasting and do not design for the deterministic and stable rollouts required for such tasks.

In this work, we propose \gls{tempo}, a latent flow matching capable of forecasting physically meaningful fields over long time horizons with high fidelity in both spatial and spectral characteristics. We perform sparse conditioning for added computational- and data- efficiency, and channel folding to process spatiotemporal 3D data using conventionally 2D frameworks: We leverage recent advances in scientific machine learning by designing time-conditioned parameter-efficient shared Fourier layers within the vector field regressor, allowing for strong capture of global and local spatial modes. We derive theoretical error bounds that characterize the efficiency and expressivity of \gls{tempo}, and showcase its performance on \gls{pde} benchmarking datasets accompanied with a spectral analysis showing a distinct advantage in capturing the essential dynamics required for forecasting. We see a 16\% lower error when predicting vorticity of 2D incompressible Navier Stokes, with Pearson correlations remaining above 0.98 for a 40 step forecasting horizon, demonstrating its stable temporal forecasting and high quality generation capability.

\section{Relevant Works}

Interest in machine learning for physical systems has surged, with generative models being adapted for such tasks and borrowing features for broader generation. For example, \citet{liuDiffFNODiffusionFourier2025,liFourierNeuralOperator2021} integrate an \gls{fno} into a score-matching denoising network, leveraging its resolution-invariant properties to achieve state-of-the-art superresolution. Similarly, Fourier Neural \gls{ode} combines Fourier analysis with Neural \glspl{ode}, outperforming the original \gls{fno}, DeepONet~\citep{luDeepONetLearningNonlinear2021}, and Physics Informed Neural Networks~\citep{raissiPhysicsinformedNeuralNetworks2019} for predicting time instances~\citep{liFourierNeuralODEs2024}. Operator learning has also been integrated with \acrlong{gan}s to generalize to infinite-dimensional function spaces~\citep{rahmanGenerativeAdversarialNeural2022}. However, such approaches leverage desirable representation characteristics of Fourier embedded processing, which diverges from the focus of this work on spatio-temporal generation.
% Machine learning methods for modelling physical systems have also surged in recent years, with generative models both being repurposed for such tasks as well as borrowing desirable characteristics to improve in broader generation tasks.~\citet{liuDiffFNODiffusionFourier2025,liFourierNeuralOperator2021} is a prime example of integrating an \gls{fno} as the backbone of a score-matching denoising network to leverage its resolution invariant characteristics, demonstrating state-of-the-art results in superresolution. Fourier Neural \gls{ode} proposes a similar modification, combining Fourier analysis with Neural \glspl{ode} which outperform the original \gls{fno} and similar methods like the DeepONet~\citep{luDeepONetLearningNonlinear2021} and Physics Informed Neural Network~\citep{raissiPhysicsinformedNeuralNetworks2019} when predicting time instances~\citep{liFourierNeuralODEs2024}. Operator learning has also seen integration to \glspl{gan} for generalisation to infinite-dimensional function spaces~\citep{rahmanGenerativeAdversarialNeural2022}. However, such approaches leverage desirable representation characteristics of Fourier embedded processing, which diverges from our focus on spatio-temporal generation.

Application-specific models for scientific data have also seen development: GenCFD~\citep{molinaroGenerativeAIFast2025} proposes a conditional diffusion model to generate the underlying distributions of high fidelity flow fields. ~\citet{kerriganFunctionalFlowMatching2023} propose the first extension of \glspl{fno} to flow matching tasks and predicts plausible fluid dynamic fields. ~\citet{yaoGuidedDiffusionSampling2025} leverages neural operators in an unconditional diffusion model to improve efficiency and sees state-of-the-art performance for multi-resolution \gls{pde} tasks, as compared to its competitor DiffusionPDE~\cite{huangDiffusionPDEGenerativePDESolving2024} which originally demonstrated strong performance in solving \glspl{pde} with partial observations. Such methods have thus far focused on single-timeframe prediction, i.e., solving slices of 2D dynamic \glspl{pde}, rather than temporal rollouts as investigated here. Localized \glspl{fno} have been proposed for complex 3D geometries~\citep{flourisLocalizedFNOSpatiotemporal2025}.

Models designed to predict sequences of future states include the aforementioned large-scale Chronos and GenCast~\citep{ansariChronosLearningLanguage2024,priceProbabilisticWeatherForecasting2025}. In addition, pyramidal flow matching~\citep{jinPyramidalFlowMatching2025} produces state-of-the-art video generation compared to leading models~\citep{zhengOpenSoraDemocratizingEfficient2024}, representing a successful flow matching foundation model.~\citep{tamirConditionalFlowMatching2024} present conditional flow matching for time series, succeeding in long 1D trajectories where neural \glspl{ode} fail, but has not scaled to 2D spatiotemporal data.~\citet{kolloviehFlowMatchingGaussian2024} extends this with Gaussian processes for forecasting tasks outside of scientific machine learning. We focus instead on models that fall between these two categories, scaling reasonably to 2D data to match common \gls{pde} settings.

% \vspace{-1em}

\section{Method}

We begin by developing the background which is then used to construct our method. Flow matching learns a time-dependent velocity field $v_\theta(z,t)$ defining an ODE in the latent space:
\begin{equation}\label{eq:latentode}
\frac{dz(t)}{dt} = v_\theta(z(t),t), \quad z(0) \sim \pi_0,
\end{equation}
where $\pi_0$ is a simple prior (e.g., Gaussian). Integrating this \gls{ode} transports samples to the latent data distribution $\pi_1$, see Appendix~\ref{app:fm}. Training reduces to a regression objective that matches the model velocity field to a target velocity along interpolating probability paths~\citep{lipmanFlowMatchingGenerative2023}. This enables deterministic, simulation-free sampling from complex distributions.

\begin{table}[h]
\centering
\resizebox{\linewidth}{!}{
\begin{threeparttable}
\caption{Representative Path Choices in Flow Matching Models.}
\label{tab:paths}
\begin{tabular}{lllll}
\toprule
Path & $a_t$ & $b_t$ & $c_t$ & Parameter definitions \\
\midrule
\acrshort{affine}\tnote{1}
& $t$ & $0$ & $(1 - (1 - \epsilon_{\min})t)^2$ 
& $\epsilon_{\min} \geq 0$: min. noise level \\
\acrshort{river}\tnote{2} 
& $(1 - (1 - \sigma_{\min})t)$ & $t$ & $\sigma^2$ 
& $\sigma \geq 0$: noise scale, $\sigma_{\min} \geq 0$: min. noise \\
\acrshort{slp}\tnote{3} 
& $(1-t)$ & $t$ & $\sigma_{\min}^2 + \sigma^2 t(1-t)$ 
& $\sigma, \sigma_{\min} \geq 0$: variance parameters \\
\acrshort{ve}\tnote{4}
& $1$ & $0$ & $\sigma_t^2$ 
& $\sigma_t$: geometric schedule, $\sigma_{\min}, \sigma_{\max} > 0$ \\
\acrshort{vp}\tnote{4}
& $\exp(-\tfrac{1}{2}T(1-t))$ & $0$ & $1 - \exp(-T(1-t))$ 
& $\beta_{\min}, \beta_{\max} > 0$, $T(t) = \int_0^t \beta(s)\,ds$ \\
\bottomrule
\end{tabular}
\tiny\begin{tablenotes}
\item[1] ~\citep{lipmanFlowMatchingGenerative2023}, \textsuperscript{2} \citep{davtyanEfficientVideoPrediction2023}, \textsuperscript{3}\citep{limElucidatingDesignChoice2025},\textsuperscript{4} \citep{ryzhakovExplicitFlowMatching2024}
\end{tablenotes}
\end{threeparttable}
}
\end{table}

A key component of flow matching is the choice of the probability density path $p_t$ interpolating between the reference distribution $\pi_0$ and the target $\pi_1$. We focus on Gaussian conditional paths with closed-form velocity fields:  
\[
p_t\!\left(Z \mid \tilde{Z} := (Z_0, Z_1)\right) 
= \mathcal{N}\Big(Z \,\big|\, a_t Z_0 + b_t Z_1, \; c_t^2 I \Big),
\]  
where $a_t, b_t, c_t$ define the path (\cref{tab:paths}). This pair-conditional path is defined for a specific transition $(Z_0,Z_1)$, and the marginal interpolant is obtained by averaging over all pairs: $p_t(Z) = \mathbb{E}_{(Z_0,Z_1)}[p_t(Z \mid Z_0,Z_1)]$. While $\pi_0$ is typically a standard Gaussian, intermediate densities $p_t$ can follow diffusion-inspired, optimal transport, or other custom schedules.

To parameterize $v_\theta$, we modify \glspl{fno}, which approximate mappings between functions via spectral convolution layers. Given input $u$, the \gls{fno} parameterizes an operator  as $\mathcal G_\theta: u \mapsto \tilde{u}, \quad \tilde{u}: \mathcal{D} \to \mathbb{R}^{c_\text{out}}$, that maps $u$ to an output function $\tilde{u}$. Iterative Fourier layers perform spectral transformations of the input $\hat{u}(k) = \mathcal{F}[u](k), \quad \hat{\tilde{u}}(k) = R_\theta(k) \cdot \hat{u}(k),$ followed by an inverse Fourier transform back to the spatial domain; $\tilde{u}(x) = \mathcal{F}^{-1}[\hat{\tilde{u}}](x),$ with $R_\theta(k)$ being learnable Fourier-mode weights and $\mathcal{F}$ denoting the Fourier transform. This spectral representation allows the \gls{fno} to efficiently capture long-range dependencies and global correlations in the data.

\subsection{\acrfull{tempo}}

Using an \gls{fno}-inspired regressor to learn the vector field of a flow matching model has a number of benefits, namely, the added expressivity that the Fourier representation provides at a low computational cost thanks to highly optimised \gls{fft} operations. Building on prior analysis of \glspl{fno} for solving \glspl{pde} \citep{kovachkiUniversalApproximationError2021}, we show that an FNO-inspired regressor can achieve an upper bound on approximation error for flow matching models and we provide a lower bound on the accuracy achievable by sampler-based methods (e.g., Transformer or U-Net) in relation to their number of parameters.

\begin{theorem}[FNO regressor constructive upper bound]
\label{thm:FNO_upper}
Let $\mathbb T^d$ be the $d$--torus. Fix $s,s'\ge0$ and let 
$\mathcal U\subset H^s(\mathbb T^d)$ be compact. 
Suppose $\mathcal G:\mathcal U\to H^{s'}(\mathbb T^d)$ is continuous and satisfies 
$|\widehat{\mathcal G(u)}(k)|\le C_\lambda(1+|k|)^{-p}$ 
for all $u\in\mathcal U$, $k\in\mathbb Z^d$, with constants $C_\lambda>0$, $p>0$. 
If $p> s'+\tfrac d2$ and we define $\alpha:=p-s'-\tfrac d2>0$, then for every 
$\varepsilon>0$ there exists a Fourier Neural Operator $\mathcal G_\theta$ with
\[
P_{\mathrm FNO}(\varepsilon)\;\lesssim\;\varepsilon^{-\,d/\alpha},
\]
such that $\sup_{u\in\mathcal U}\|\mathcal G(u)-\mathcal G_\theta(u)\|_{H^{s'}}\le\varepsilon$.  The hidden constants depend only on 
$d,s,s',\mathcal U,C_\lambda$ and mild/logarithmic factors.
\end{theorem}

This result is in line with the estimates and arguments made in \citep{kovachkiUniversalApproximationError2021}.

\begin{proof}[Sketch of proof of Theorem~\ref{thm:FNO_upper}]
(Spectral truncation.) The Fourier decay assumption implies that high-frequency 
modes of $\mathcal G(u)$ contribute at most $O(K^{-2\alpha})$ to the $H^{s'}$-error. 
Choosing $K\asymp \varepsilon^{-1/\alpha}$ makes this truncation error $\le \varepsilon/2$.  

(Finite-dimensional reduction.) For this cutoff $K$, the operator $\mathcal G_K$ 
is determined by $O(K^d)$ Fourier coefficients, and inputs can likewise be 
restricted to finitely many low modes without significant loss of accuracy. 
Thus the problem reduces to approximating a continuous map between compact 
subsets of $\R^{m_{\rm in}}$ and $\R^{m_{\rm out}}$, with 
$m_{\rm out}\asymp K^d$.  

(Approximation by networks.) Standard universal approximation results (or the 
constructive \gls{fno} design in \citep{kovachkiUniversalApproximationError2021}) ensure that such a finite map 
can be uniformly approximated by a network with $O(K^d)$ parameters, up to mild 
logarithmic factors.  

(Conclusion.) Combining these errors yields an overall accuracy $\varepsilon$ 
with parameter count $P\lesssim K^d \asymp \varepsilon^{-d/\alpha}$, proving the claim.
\end{proof}

\begin{proposition}[Transformer/UNet Sampler-based lower bound]
\label{prop:sampler_lower}
Under the assumptions of Theorem~\ref{thm:FNO_upper}, consider any learner that 
observes each $u\in\mathcal U$ only through $n$ fixed point evaluations and 
applies a parametric map with $P$ real parameters, required in the worst case 
to reconstruct all Fourier modes up to radius $K\asymp\varepsilon^{-1/\alpha}$. 
Then necessarily
\[
n \;\gtrsim\; \varepsilon^{-\,d/\alpha}, 
\qquad 
P_{\mathrm{sampler}}(\varepsilon)\;\gtrsim\;\varepsilon^{-\,\beta d/\alpha},
\]
for some architecture--dependent $\beta\ge1$ (optimistically $\beta=1$ when only 
diagonal mode-wise maps are needed, generically $\beta=2$ for arbitrary dense 
linear maps). These bounds are information-theoretic and asymptotic, up to 
constants and mild/logarithmic factors.
\end{proposition}

\begin{proof}[Sketch of proof of Proposition~\ref{prop:sampler_lower}]
(Sampling necessity.) The $K$--mode subspace $V_K$ has dimension $D_K\asymp K^d$.
Sampling at $n$ points defines a linear map $S:V_K\to \mathbb{C}^n$. For $S$ to be injective
on $V_K$, its matrix must have rank $D_K$, hence
$n\ge D_K\asymp K^d$.  

(Parameter complexity.) After sampling, the learner implements a parametric map
$M: \mathbb{C}^n\to \mathbb{C}^m$. To represent arbitrary linear maps on the $D_K$-dimensional
coefficient space (e.g. arbitrary diagonal multipliers), the parameter family
must have at least $P\gtrsim D_K$ degrees of freedom. For fully general dense
linear maps one needs $P\gtrsim D_K^2$.  

(Conversion.) Substituting $K\asymp \varepsilon^{-1/\alpha}$ (from the theorem)
gives $n\gtrsim \varepsilon^{-d/\alpha}$ and 
$P\gtrsim \varepsilon^{-\beta d/\alpha}$ with $\beta=1$ (optimistic) or $\beta=2$
(dense case), establishing the lower bound, see Appendix~\ref{app:proofs} for the extended proof.
\end{proof}

\begin{corollary}[FNO vs sampler scaling]\label{cor:efficiency}
From Theorem~\ref{thm:FNO_upper} and Proposition~\ref{prop:sampler_lower} one has
\[
P_{\mathrm FNO}(\varepsilon)\;\lesssim\;\varepsilon^{-\,d/\alpha},
\qquad 
P_{\mathrm{sampler}}(\varepsilon)\;\gtrsim\;\varepsilon^{-\,\beta d/\alpha}.
\]
Hence, whenever $\beta>1$, \glspl{fno} achieve the same accuracy $\varepsilon$ with
asymptotically fewer parameters than sampler--based learners.
\end{corollary}

\paragraph{\gls{tempo}}
Consequently, we propose a novel generation model which capitalises on the \gls{fno}'s expressivity and capacity to model complex velocity fields  by designing a latent time-conditioned \gls{fno} vector field regressor using channel folding for both efficiency and enhanced temporal coherency. Together with temporal conditioning~\citep{davtyanEfficientVideoPrediction2023}, these define a novel, end-to-end trainable model for predicting latent dynamics.

Let $f_\phi: \mathbb{R}^X \rightarrow \mathbb{R}^Z$ denote an encoder mapping data points $x$ to latent embeddings $z = f_\phi(x)$. We can then define a latent-space velocity field described by~\ref{eq:latentode} where $v_\theta$ is parameterized by an \gls{fno}.

To capture the temporal dependencies, we leverage \emph{sparse conditioning}~\citep{davtyanEfficientVideoPrediction2023,limElucidatingDesignChoice2025}. For some discrete-time sequence $\{x_t\}_{t=1}^N$ with $x_t \in \mathcal{X}$, its latent representation is given by $\{z_t\}_{t=1}^N$, where $z_t = f_\phi(x_t)$. 
For a prediction horizon $T \in \{L, \dots, N-1\}$ with sequence length $L$, the objective is to predict the next latent embedding $z_{T+1}$. We define a reference embedding to be $z_T$, corresponding to the most recent observation prior to the prediction target, and a conditioning embedding as some observation selected at a timestep $\tau \in \{T-L, \dots, T-1\}$. These two embeddings are concatenated with the temporal offset, defined as $\Delta = T - \tau$, which is the extent of temporal data the model is provided to predict the next-step embedding, $\hat{z}_{T+1} = f_\theta(z_T, z_\tau, \Delta)$.

To process the spatiotemporal input data and conditioning while preserving compatibility the 2D \gls{fno}, we then propose a \emph{channel folding} scheme that merges the batch and channel dimensions (as opposed to the more conventional batch and time dimensions) to align with the original input ordering of the \gls{fno}). To match with the expected inputs of the form \(\mathbb{R}^{B' \times T' \times H \times W}\), we collapse the batch and channel axes into a single ``effective batch'' dimension $u' \in \mathbb{R}^{(B \cdot C) \times T \times H \times W}$. This folding operation effectively treats each channel of each sample as an independent element within the extended batch. As a consequence, the \gls{fno} is applied identically across all channels but without cross-channel mixing at this stage.

This \emph{time-conditioned \gls{fno}} then operates over latent temporal embeddings as functions on their spatial domain $v_\theta(z,t) = \mathcal G_\theta(z)$ to learn the time-dependent vector field that transports a reference latent distribution $\pi_0$ to the latent data distribution $\pi_1$. By leveraging the spectral inductive bias of \glspl{fno}, the learned velocity field can capture both local and long-range correlations efficiently, improving the expressivity and stability of flow matching in high-dimensional latent spaces.

% A central component of flow matching is the choice of the probability density path $p_t$ that interpolates between the reference distribution $\pi_0$ and the target distribution $\pi_1$. We focus on the family standard Guassian conditional probability paths which produce closed-form expressions for the velocity field and the conditional flow. Generally, the these are expressed in the form
% $$p_t\!\left(Z \mid \tilde{Z} := (Z_0, Z_1)\right) 
% = \mathcal{N}\!\Big(Z \,\Big|\, a_t Z_0 + b_t Z_1, \; c_t^2 I \Big),$$
% where $a_t, b_t, \text{and} \  c_t$ can be defined to design different paths, \cref{tab:paths}. Note that $p_t(Z \mid Z_0, Z_1)$ denotes a pair-conditional path, defined given a specific 
% transition pair $(Z_0,Z_1)$. The marginal interpolant $p_t(Z)$ between $\pi_0$ and $\pi_1$ is 
% obtained by averaging over all such pairs, i.e.\ $p_t(Z) = \mathbb{E}_{(Z_0,Z_1)}[p_t(Z \mid Z_0,Z_1)]$. While $\pi_0$ is typically fixed to be a standard Gaussian, the intermediate densities $p_t$ 
% can be constructed in multiple ways, ranging from diffusion-inspired paths to 
% optimal transport formulations or other custom schedules.

\section{Experiments}

The \gls{tempo} is evaluated with the goal of assessing its ability to learn accurate stochastic latent-space dynamics and forecast high-dimensional solution fields over medium to long time horizons. We test our method over \gls{pde} datasets which pose challenging spatio-temporal correlations and multiscale features, making them a natural testbed for latent flow-based modeling.

% \textcolor{red}{If this paragraph can be sharpened a bit to make it easier to assimilate:}
Our proposed \gls{tempo} was set against five key methods. The state-of-the-art video generation method based on a U-Net shaped \gls{vit} and modified optimal transport path \gls{river} proposed by \citet{davtyanEfficientVideoPrediction2023} matches or surpasses common video prediction benchmarks using ~10x fewer computational resources~\citep{davtyanEfficientVideoPrediction2023}. We also include the baseline conditional flow matching~\citet{lipmanFlowMatchingGenerative2023} which implements a U-Net trained using a theoretically optimal \gls{affine} path. The \gls{slp} was proposed by~\citet{limElucidatingDesignChoice2025}, tested with a \gls{vit} to directly address the challenges of spatiotemporal forecasting for \gls{pde} datasets. The Transformer-based latent space flow matching method with \gls{affine} proposed by~\citet{daoFlowMatchingLatent2023} further demonstrates competitive performance in image generation using latent flow matching compared against both flow matching models and diffusion models~\citep{phungWaveletDiffusionModels2023,hoDenoisingDiffusionProbabilistic2020} among others. We also evaluate both \gls{vp} and \gls{ve} paths which generalise the Denoising Diffusion Probabilistic noise perturbation model and a score-based model to flow matching paths, respectively~\citep{hoDenoisingDiffusionProbabilistic2020,songScoreBasedGenerativeModeling2021}. \citet{ryzhakovExplicitFlowMatching2024} establishes strong theoretical backing for both paths.
% , and outperforms optimal transport in common image generation benchmarks when tested with a U-Net regressor.

% Our proposed \gls{tempo} was set against the following methods:
% \begin{itemize}
%     \item \gls{river}~\citep{davtyanEfficientVideoPrediction2023} which implements [Transformer, \gls{river} probability path]
%     \item Conditional Flow Matching~\citep{lipmanFlowMatchingGenerative2023} which implements [U-Net, \gls{affine}]
%     \item \Gls{ve}~\citep{songScoreBasedGenerativeModeling2021a} which was originally designed for a score-based model and provides a generalisation of noise perturbations used in score matching with Langevin dynamics. \citep{limElucidatingDesignChoice2025} implements these with the Transformer vector field regressor.
%     \item \Gls{vp}~\citep{songScoreBasedGenerativeModeling2021a} which likewise provides a generalisation of noise perturbations used in \glspl{ddpm}~\citep{hoDenoisingDiffusionProbabilistic2020}. \citep{limElucidatingDesignChoice2025} implements these with the Transformer vector field regressor.
% \end{itemize}

We then ablate the specific implementation of the methods (consisting of a specific architecture and a specific probability path). In summary, the choice of regressor includes our proposed \gls{tempo} regressor, and additionally implement a \gls{vit} regressor~\citep{davtyanEfficientVideoPrediction2023,limElucidatingDesignChoice2025} and a classic U-Net regressor~\citep{lipmanFlowMatchingGenerative2023}. We pretrain a convolutional autoencoder with residual and attention blocks to obtain a compressed latent representation of the dynamics, see Appendix~\ref{app:ae}.

All methods were conditioned using sparse conditioning. These models are then supervised by each probability density paths described in~\cref{tab:paths}, with further details in Appendix~\ref{app:modelhypers}. The Adam optimiser was used with a learning rate of 1e-4 for the \gls{fno}, and 5e-5 for the \gls{vit} and U-Net regressors. Models are trained on an 80/10/10 training to validation to test data splits. 

We evaluate our models on three spatiotemporal \gls{pde} datasets: the \acrfull{swe}, which simulate 2D free-surface flows; \acrfull{rd} exhibiting multiscale nonlinear patterns; and \acrfull{nsv} dataset capturing chaotic turbulent dynamics. During training, models are sparsely conditioned on the first 15 frames and tasked with predicting the subsequent frame at resolutions of $1\times128\times128$ (\gls{swe}), $2\times128\times128$ (\gls{rd}), and $1\times64\times64$ (\gls{nsv}), see Appendix \ref{app:datasets}.
% We consider the following PDE datasets in our experiments.

% \emph{\Acrfull{swe}} The shallow water equations are derived from compressible Navier Stokes equations and model free-surface flow problems, described by a set of 2D hyperbolic \glspl{pde}. The model is sparsely conditioned within the first 15 frames and predicts the following frame at a resolution of $1\times 128 \times 128$.

% \emph{\Acrfull{rd}} The 2D reaction diffusion dataset models a system of two non-linearly coupled variables which exhibit features across multiple spatial and temporal scales. The model is sparsely conditioned within the first 15 frames and predicts the following frame at a resolution of $2\times 128 \times 128$ representing velocity in the $x$ and $y$ directions.

% \emph{\Acrfull{ns2c}} The 2D incompressible Navier Stokes dataset models the dynamics of incompressible fluid flow, which occur in dynamics munch lower than the speed of propagation of wades in the medium. The model is sparsely conditioned within the first 15 frames and predicts the following frame at a resolution of $2\times 512 \times 512$ representing velocity in the $x$ and $y$ directions.

% \emph{\acrfull{nsv}} This dataset is a variation of 2D incompressible Navier Stokes expressed instead as vorticity, or the curl of the velocity field $\omega = \nabla \times v$, and was configured to produce chaotic dynamics characterics of turbulent flows. The model is sparsely conditioned within the first 15 frames and predicts the following frame at a resolution of $1\times 64 \times 64$.

\section{Results}\label{sec:results}

\begin{table}[ht!]
\centering
\resizebox{\linewidth}{!}{
\begin{threeparttable}
\caption{\gls{nsv} Results: Comparison of \gls{tempo}, U-Net, and \gls{vit} models.}
\label{tab:nsv}
\begin{tabular}{llrrrrrrr}
\toprule
Regressor & Path & MSE $\downarrow$ & SpectralMSE $\downarrow$ & RFNE $\downarrow$ & PSNR $\uparrow$ & Pearson $\uparrow$ & SSIM $\uparrow$ \\
\midrule
\multirow{4}{*}{\gls{tempo}}
  & \acrshort{river}     & \textbf{5.63e-02} & \textbf{3.84e-02} & \textbf{2.50e-01} & \textbf{25.19} & \textbf{0.969} & 0.786 \\
  & \acrshort{affine}    & 5.77e-02 & 3.98e-02 & 2.54e-01 & 25.08 & 0.968 & \textbf{0.789} \\
  & \acrshort{vp}        & 8.10e-02 & 5.34e-02 & 2.85e-01 & 23.61 & 0.955 & 0.731 \\
  & \acrshort{ve}        & 2.96e-01 & 1.73e-01 & 5.60e-01 & 17.98 & 0.821 & 0.373 \\
\midrule
\multirow{4}{*}{\gls{vit}}
  & \acrshort{affine}\tnote{1}    & 6.75e-02 & 4.38e-02 & 2.72e-01 & 24.40 & 0.962 & 0.758 \\
  & \acrshort{river}\tnote{2} & 6.88e-02 & 4.33e-02 & 2.73e-01 & 24.32 & 0.962 & 0.750 \\
  & \acrshort{vp}\tnote{3}    & 7.77e-02 & 4.95e-02 & 2.89e-01 & 23.79 & 0.956 & 0.729 \\
  & \acrshort{ve}\tnote{3}    & 1.63e+00 & 9.27e-01 & 1.35e+00 & 10.67 & 0.118 & 0.024 \\
\midrule
\multirow{4}{*}{U-Net}
  & \acrshort{vp}\tnote{4}        & 4.05e-01 & 3.26e-01 & 6.71e-01 & 16.62 & 0.756 & 0.323 \\
  & \acrshort{river}     & 4.08e-01 & 3.28e-01 & 6.74e-01 & 16.59 & 0.752 & 0.321 \\
  & \acrshort{affine}\tnote{5} & 4.10e-01 & 3.42e-01 & 6.76e-01 & 16.57 & 0.751 & 0.324 \\
  & \acrshort{ve}\tnote{4}        & 5.02e-01 & 3.70e-01 & 7.48e-01 & 15.68 & 0.694 & 0.263 \\
\bottomrule
\end{tabular}
\tiny\begin{tablenotes}
\item[1] \citep{daoFlowMatchingLatent2023}, \textsuperscript{2}\citep{davtyanEfficientVideoPrediction2023}, \textsuperscript{3} \citep{limElucidatingDesignChoice2025,songGenerativeModelingEstimating2020}, \textsuperscript{4}\citep{ryzhakovExplicitFlowMatching2024}, \textsuperscript{5}\citep{lipmanFlowMatchingGenerative2023}
\end{tablenotes}
\end{threeparttable}
}
\end{table}

\begin{table}[h!]
\centering
\resizebox{\linewidth}{!}{
\begin{threeparttable}
\caption{\gls{swe} and \gls{rd} Results: Comparison of \gls{tempo}, U-Net, and \gls{vit} models.}
\label{tab:swe_dr}
\begin{tabular}{lllllllll}
\toprule
Dataset & Method & MSE $\downarrow$ & SpectralMSE $\downarrow$ & RFNE $\downarrow$ & PSNR $\uparrow$ & Pearson $\uparrow$ & SSIM $\uparrow$ \\
\midrule
\multirow{5}{*}{\gls{swe}}
 & \method{\acrshort{tempo}}{\acrshort{affine}}       & \textbf{6.64e-05} & \textbf{5.65e-05} & \textbf{7.64e-03} & \textbf{46.5} & \textbf{0.998} & \textbf{0.997} \\
 & \method{\acrshort{vit}}{\acrshort{affine}}\tnote{1}     & 9.59e-05 & 7.93e-05 & 9.06e-03 & 44.9 & 0.997 & 0.995 \\
 & \method{\acrshort{vit}}{\acrshort{vp}}\tnote{2}    & 1.30e-04 & 8.81e-05 & 1.05e-02 & 43.6 & 0.996 & 0.993 \\
 & \method{\acrshort{vit}}{\acrshort{river}}\tnote{3} & 2.99e-04 & 1.67e-04 & 1.63e-02 & 40.0 & 0.992 & 0.981 \\
 & \method{\acrshort{vit}}{\acrshort{slp}}\tnote{4}  & 6.60e-04 & - & 1.28e-01 & 36.1 & - & 0.93 \\
\midrule\midrule
\multirow{4}{*}{\gls{rd}}
 & \method{\acrshort{tempo}}{\acrshort{affine}}     & \textbf{2.76e-05} & \textbf{2.18e-05} & \textbf{3.29e-02} & \textbf{65.7} & \textbf{1.000} & \textbf{0.999} \\
 & \method{U-Net}{\acrshort{affine}}\tnote{5}  & 3.09e-05 & 2.45e-05 & 3.57e-02 & 65.2 & 0.999 & \textbf{0.999} \\
 & \method{\acrshort{vit}}{\acrshort{affine}}     & 6.30e-04 & 4.40e-04 & 1.67e-01 & 52.2 & 0.987 & 0.986 \\
 & \method{\acrshort{vit}}{\acrshort{slp}}\tnote{4} & 3.56e-04 & - & 1.16e-01 & 34.3 & -  & 0.90 \\
\bottomrule
\end{tabular}
\tiny\begin{tablenotes}
\item[1] \citep{daoFlowMatchingLatent2023}, \textsuperscript{2}\citep{limElucidatingDesignChoice2025,songGenerativeModelingEstimating2020}, \textsuperscript{3}\citep{davtyanEfficientVideoPrediction2023}, \textsuperscript{4}\citep{limElucidatingDesignChoice2025}; results reported from original paper trained on same dataset., \textsuperscript{5}\citep{lipmanFlowMatchingGenerative2023}
\end{tablenotes}
\end{threeparttable}
}
\end{table}

Overall, \gls{tempo} outperforms the methods proposed by \citet{limElucidatingDesignChoice2025,songDenoisingDiffusionImplicit2022,lipmanFlowMatchingGenerative2023} and \citet{davtyanEfficientVideoPrediction2023} as well as the ablated methods. For results predicting \gls{nsv} in~\cref{tab:nsv}, we observe a 16\% improvement in MSE and an 11.4\% lower spectral MSE, producing spatially and spectrally accurate next steps. Its lower RFNE indicates reduced scale-independent error, while SSIM shows improved fidelity in local features, critical for the localized vorticity patterns where small spatial distortions significantly affect downstream evolution~\citep{majdaVorticityIncompressibleFlow2001}. PSNR and Pearson see lower normalised ranges in values, indicating that large scale features like the vorticity intensity and global structure agreement, respectively, are more easily captured across all models, with a clear advantage by \gls{tempo}; additional visualisations in Appendix~\ref{app:nsv}.

We select top performing comparisons for \gls{swe} and \gls{rd}, ~\cref{tab:swe_dr}), where \gls{tempo} maintains superior performance. In \gls{swe}, it achieves a 28.8\% lower SpectralMSE and higher PSNR, indicating faithful amplitude, spectral content, and structural coherence with sharp boundaries preserved, see Appendix~\ref{app:swe} for additional visualisations and ablated comparisons. Overall MSE is reduced by 30.8%.

In \gls{rd}, \method{U-Net}{\gls{affine}} competes closely with \gls{tempo}, benefiting from translation-equivariant convolutional layers that capture multi-scale dynamics and repeating local structures~\citep{cohenGroupEquivariantConvolutional2016}. Both \gls{tempo} and the U-Net have nearly matched PSNR, Pearon, and SSIM scores, with an improvement of 11\% in SpectralMSE from the \gls{tempo}. By contrast, the next best \gls{vit} regressor model is 95.6\% drop in SpectralMSE, where attention might emphasize low-frequency global structures~\citep{wangAntiOversmoothingDeepVision2022,piaoFredformerFrequencyDebiased2024}; see visual comparison in Appendix~\ref{app:rd}.

% \begin{wrapfigure}[18]{l}{0.6\textwidth}
% \vspace{-2em}
%     \includegraphics[width=\linewidth]{images/pearson_nsv.png}
%     \caption{\textbf{Pearson Correlation over Forecasted \gls{nsv} Timesteps:} A timeseries was generated from an initial 10 available timesteps and then autoregressively evolved, which were then compared with the ground truth timeseries using Pearson correlation.}
%     \label{fig:pearson_nsv}
% \end{wrapfigure}

% \begin{figure}[h!]
%     \centering
%     \begin{subfigure}{0.73\textwidth}
%         \includegraphics[width=\linewidth]{images/pearson_nsv.png}
%         \caption{
% \textbf{Pearson correlation over forecasted \gls{nsv} timesteps:} 
% Forty timesteps are forecasted by \gls{tempo}, \gls{vit}, and U-Net. 
% The Pearson correlation coefficient shows significant degradation for the U-Net, 
% oscillatory behavior and degradation for the \gls{vit}, and consistently stable values above $0.98$ for \gls{tempo}.}
%         \label{fig:pearson_nsv}
%     \end{subfigure}%
%     \hfill
%     \begin{subfigure}{0.25\textwidth}
%         \raisebox{0\height}{%
%             \includegraphics[width=\linewidth]{images/snapshots.png}
%         }
%         \caption{\textbf{Example timesteps of \gls{tempo}, \gls{vit}, and U-Net}: a) true data, b) \gls{tempo}, c) \gls{vit}, d) U-Net: Forecasted timesteps 5, 15, and 35 demonstrate the that the \gls{vit} and U-Net models clearly diverge, with U-Net regressing to a noisy mean.}
%         \label{}
%     \end{subfigure}
% \end{figure}

\begin{figure}[h!]
    \centering
    \begin{subfigure}{0.68\textwidth}
        \includegraphics[width=\linewidth]{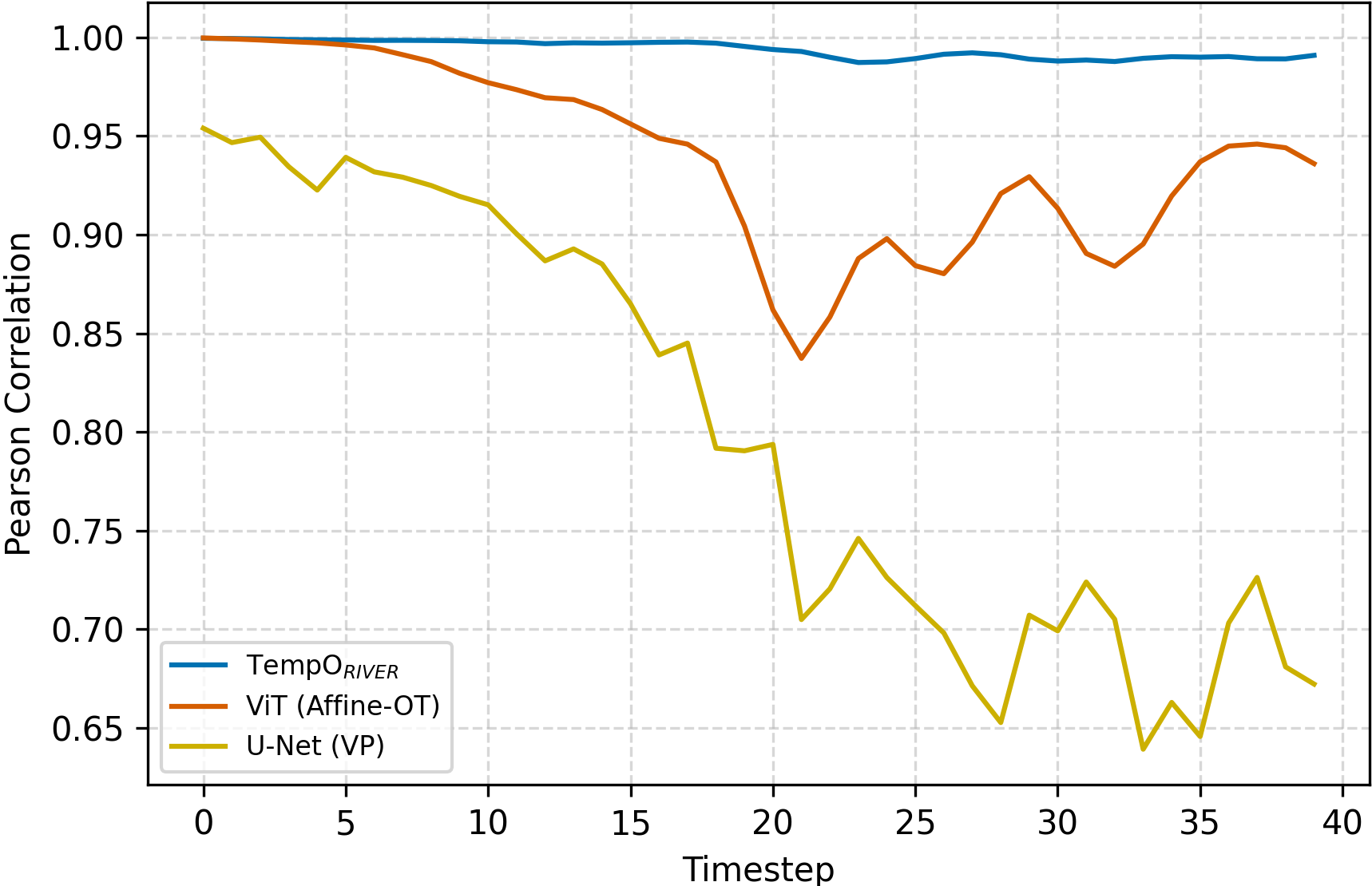}
        \label{fig:pearson_nsv}
    \end{subfigure}%
    \hfill
    \begin{subfigure}{0.3\textwidth}
        \raisebox{+0.1\height}{% lift by 10%
            \includegraphics[width=\linewidth]{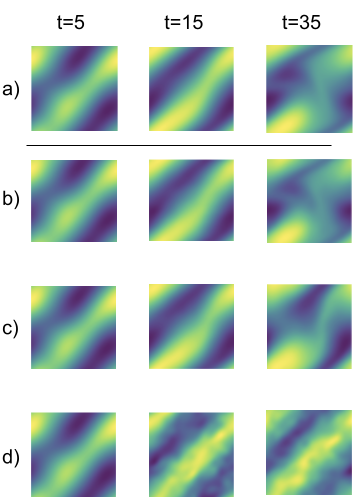}
        }
        \label{fig:snapshots}
    \end{subfigure}

    \vspace{-1.0em} % tighten space before caption
    \caption{\textbf{Prediction performance comparison for \gls{nsv}.} 
    \textit{Left: Pearson correlation across forecasted timesteps.} 
    Forty timesteps are predicted by \gls{tempo}, \gls{vit}, and U-Net conditioned on two preceding timesteps and sampled for each proceeding step. 
    The Pearson correlation coefficient shows significant degradation for the U-Net, oscillatory behavior and degradation for the \gls{vit}, and consistently stable values above $0.98$ for \gls{tempo}. 
    \textit{Right: Predicted vorticity fields.} 
    True data (a), \gls{tempo} (b), \gls{vit} (c), and U-Net (d). 
    At timesteps 5, 15, and 35 the \gls{vit} and U-Net models clearly diverge, with U-Net regressing to a noisy, while TempO maintaining excellent accuracy.}
    \label{fig:tempo_vs_vit_unet}
\end{figure}

The timeseries forecasting task, see~\cref{fig:tempo_vs_vit_unet}, evaluates how well models capture the underlying \gls{pde}. The model is provided two initial timeframes representing the conditioning and reference frames, respectively, and is then sampled for increasing temporal offsets with the reference set to be the most recent generation. \gls{tempo} maintains Pearson correlation above 0.98 over 40 forecasted timesteps, indicating stable amplitude and phase tracking. The \gls{vit} regressor holds above 0.95 for 20 steps before degrading, while the flow matching baseline~\citep{lipmanFlowMatchingGenerative2023} shows steady decline. This suggests \gls{tempo} effectively mimics the dynamics without significant error accumulation. This is further demonstrated by visualisations of the vorticity field at key timesteps in~\cref{fig:tempo_vs_vit_unet} (right), where $t=35$ most clearly shows \gls{tempo}'s faithful capture of turbulent eddies in comparison to the \gls{vit} regressor, which fails to predict the small vortical structure.

\subsection{Spectral Analysis}

\begin{figure}
    \centering
    % \vspace{-3.5em}
    \includegraphics[width=1.0\linewidth]{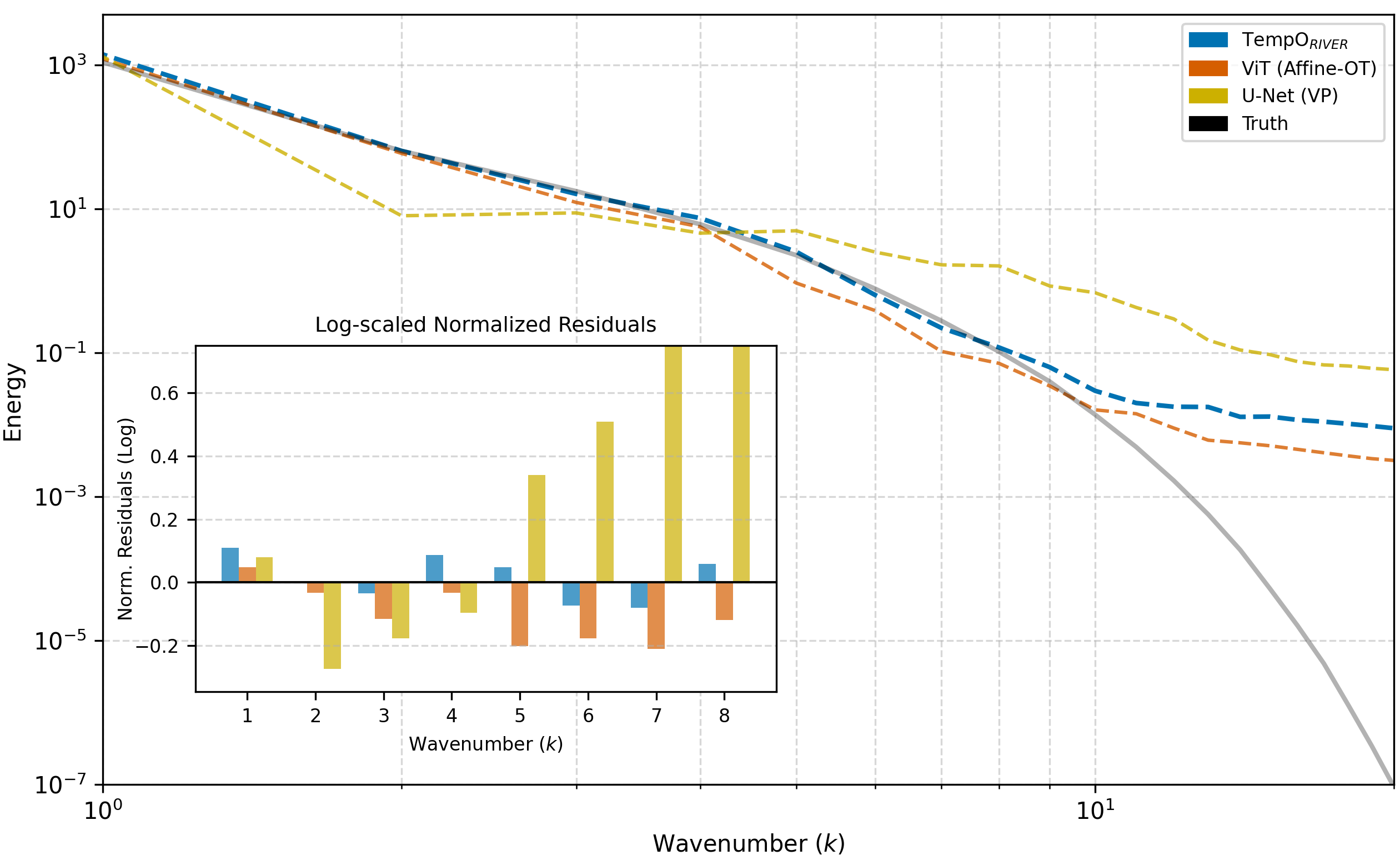}
    \caption{\textbf{Spectral graphs for \gls{nsv}.} 
    Comparison of energy spectra for \gls{tempo}, a \gls{vit}-based model, 
    and the U-Net baseline~\citep{lipmanFlowMatchingGenerative2023}. 
    The first eight Fourier modes capture 99\% of the energy, with higher modes negligible. 
    \gls{tempo} aligns closely, while the \gls{vit} underestimates energy beyond $k=4$. 
    The inset bar plot shows \gls{tempo} oscillating tightly around zero with small deviations, the \gls{vit} producing larger negative deviations, 
    and the U-Net performing markedly worse.}
    \label{fig:spectrum}
\end{figure}

% \begin{figure}[t]
% \centering
% \begin{minipage}{0.68\textwidth}
% \hspace{-2em}
%     \includegraphics[width=1.2\linewidth]{images/spectrum_nsv.png}
% \caption{\textbf{Spectral graphs for  for \gls{nsv}.} 
% Comparison of energy spectra for \gls{tempo}, a \gls{vit}-based model, 
% and the U-Net baseline~\citep{lipmanFlowMatchingGenerative2023}. 
% The first eight Fourier modes capture 99\% of the energy, with higher modes negligible. 
% \gls{tempo} aligns closely across all modes, while the \gls{vit} underestimates energy beyond $k=4$. 
% The inset (log-scaled residuals) shows \gls{tempo} oscillating tightly around zero with small deviations, 
% the \gls{vit} producing larger negative deviations, 
% and the U-Net performing markedly worse across scales.}

%     \label{fig:spectrum}
% \end{minipage}
% \hfill
% \begin{minipage}{0.28\textwidth}
%     \centering
%     \begin{tabular}{rr}
%     \toprule
%     Modes & SpectralMSE \\
%     \midrule
%     1  & 8.57e-02 \\
%     2  & 4.10e-02 \\
%     4  & 3.98e-02 \\
%     8  & 3.79e-02 \\
%     16 & 3.74e-02 \\
%     \bottomrule
%     \end{tabular}
%      \captionof{table}{Ablation ofFourier modes cutoff  with \gls{tempo}. \textcolor{red}{dont like the position of this table}}
    
%     % was trained over varying cutoff points for the $k$ lowest modes to keep during the \gls{fft} layers.
%     \label{tab:spectralmse_modes}
% \end{minipage}
% \end{figure}

\begin{wrapfigure}[10]{l}{0.35\textwidth}
    \centering
    \vspace{-1.5em}
    \begin{tabular}{rr}
    \toprule
    Modes & SpectralMSE \\
    \midrule
    1  & 8.57e-02 \\
    2  & 4.10e-02 \\
    4  & 3.98e-02 \\
    8  & 3.79e-02 \\
    16 & 3.74e-02 \\
    \bottomrule
    \end{tabular}
    \caption{Ablation: Fourier mode cutoffs with \gls{tempo}.}
    \label{tab:spectralmse_modes}
\end{wrapfigure}

The spectral analysis of \gls{tempo} versus the top alternative \method{\gls{vit}}{\gls{affine}} and the baseline \method{U-Net}{\gls{affine}}~\citep{lipmanFlowMatchingGenerative2023} in~\cref{fig:spectrum} examines the scale-resolved error via the energy per wavenumber $k$, or at the scale of $\frac{1}{k}$. This provides scale-resolved context to the SpectralMSE, which averages the MSE of the Fourier coefficients to a single metric. For \gls{nsv}, the first 8 modes which cumulatively make up 99\% of the total energy, beyond which the modes have negligible contributions to overall flow dynamics, see Appendix~\ref{app:spectral_ground_truth}. \gls{tempo} closely follows the true spectrum compared to both \method{\gls{vit}}{\gls{affine}} and \method{U-Net}{\gls{affine}}, though all three methods diverge past $k=8$. We observe from the inset of~\cref{fig:spectrum} that \gls{tempo} exhibits a small residual which fluctuates about 0 whereas the \method{\gls{vit}}{\gls{affine}} has a negative and increasing error: the \gls{vit} regressor tends to capture the lower wavenumbers well, but then underestimates the higher wavelengths notably after $k=4$.

We observe also that the number of modes retained during the \gls{fft} of \gls{tempo} in~\cref{tab:spectralmse_modes} follows the observation of a close spectral match up until $k=8$, where the SpectralMSE sees the most improvement; however, from 8 modes to 16 modes, the performance appears to saturate. \cref{tab:spectralmse_modes} demonstrates that up to 8 modes capture the essential dynamics, while the fundamental frequency alone is insufficient and likely under-represents necessary higher frequency components; adding more than 8 modes yields diminishing returns, matching the true spectral analysis; extended metrics support this trend in Appendix~\ref{app:ablation}. This empirical saturation beyond 8 modes is consistent with the theoretical expectation in \cref{cor:efficiency}, where \glspl{fno} are shown to achieve accuracy with asymptotically fewer parameters by leveraging only the most informative spectral modes.

\subsection{Efficiency}

Finally, we also train the models over varying sequence lengths and measuring next-step prediction error (MSE) and 40-step forecast error (MSE/time), shown in~\cref{tab:seq_len}. MSE is lowest for shorter sequences, as the model learns from fewer choices of indices for sparse conditioning during training. Conversely, MSE/time decreases with longer sequences, reflecting better long-horizon performance. Notably, \gls{tempo}'s MSE/time drops faster and plateaus lower than the \gls{vit}, indicating better data efficiency to extrapolate from the same available sequence length.
% The MSE is lowest for lower sequence lengths, which can be attributed to the sparse conditioning wherein the preceeding timeframe and a randomly sampled index of an extra timeframe within the sequence length are used to learn the next step: when the sequence length is shorter, the choice of indices is smaller and thus the learning goal is simpler. However, the opposite can be seen in the MSE/time. Longer available sequence lengths result in a lower average error over the forecasted timeseries for both models. Notably, \gls{tempo}'s MSE/time decreases more steeply and plateaus lower than the \gls{vit}, indicating better data efficiency to extrapolate from the same available sequence length.

\gls{tempo} is the most lightweight model among the three choices of regressors, with \~7x fewer parameters than the \gls{vit} and \~28x fewer than the U-Net. In addition, it sees a significantly lower memory usage compared to the \gls{vit} where attention has higher demands and the U-Net where skip-connections hold onto additional memory.

While \gls{tempo} has a moderate number of FLoating Point OPerations (FLOPs), landing between the \gls{vit} and U-Net, this may be offsetted by the \glspl{nfe} seen during the \gls{ode} integration where \gls{tempo} takes only 560 evaluations to meet the same tolerances. Beyond these empirical measures, \gls{tempo} further benefits from its shared spatial Fourier layers. 
By folding the channel dimension and truncating higher modes, the spectral convolution scales as 
$O(N^2 \log N)$, in contrast to the naive $O(N^3 \log N)$ cost of a full 3D \gls{fft}. Also for reference, a \gls{vit} layer can scale as $O(N^4)$ in 2D~\ref{prop:sampler_lower}, higher than the quasi-quadratic cost of the \gls{fno}.

\vspace{-0.5em}
\begin{wraptable}[10]{l}{0.5\textwidth} % r=right, width controls wrap size
% \vspace{-1.5em}
\centering
\resizebox{\linewidth}{!}{
\begin{tabular}{lcccc}
\toprule
Model & Params & FLOPs & Mem (MB) & NFEs \\
\midrule
\gls{tempo}     & 0.49M  & 208M  & $\sim$50 & 560 \\
\gls{vit}    & 3.39M  & 10M   & $\sim$80 & 942 \\
U-Net           & 14.0M  & 555M  & $\sim$68 & 728 \\
\bottomrule
\end{tabular}
}
\caption{Model Complexity and Efficiency: \glspl{nfe} are averaged from sampling performed for~\cref{tab:nsv} for adaptive solver \texttt{dopri5} and tolerances of 1e-5.}
\end{wraptable}

\begin{wraptable}[3]{r}{0.45\textwidth} % r=right, width controls wrap size
\vspace{-15em}
\centering
\begin{tabular}{llrr}
\toprule
Method & Seq & MSE & MSE/time \\
\midrule
\multirow{5}{*}{\gls{tempo}} 
 & 2  & 4.92e-02 & 2.70e-01 \\
 & 5  & 4.75e-02 & 3.41e-01 \\
 & 10 & 5.04e-02 & 4.94e-02 \\
 & 15 & 5.61e-02 & 3.83e-02 \\
 & 25 & 6.26e-02 & 4.22e-02 \\
\midrule
\multirow{5}{*}{\makecell{\gls{vit} \\ \tiny(\gls{affine})}}
 & 2  & 6.75e-02 & 2.71e-01 \\
 & 5  & 5.43e-02 & 3.59e-01 \\
 & 10 & 6.01e-02 & 1.49e-01 \\
 & 15 & 6.70e-02 & 4.53e-02 \\
 & 25 & 7.68e-02 & 8.56e-02 \\
\bottomrule
\end{tabular}
\caption{Ablation: Performance comparison scaling with sequence length.}
\label{tab:seq_len}
\end{wraptable}

%While flow matching was chosen thanks to its close alignment between the training objective of the vector field regressor and a true \gls{pde} operator vector field, it is then also limited in generalisability between \glspl{pde} which would have different operators fundamentally;
\section{Limitations}
Flow matching models struggle with extreme data sparsity which can distort the distributions being learned, whereas hybrid models or models with explicitly defined conservations can fall back on injected physical knowledge. Additionally, similar to other generative models, adaptations, e.g. architectural modifications, would be necessary to extend the method towards a foundational model framework. Finally, while our stable and accurate 40-step forecasting represents the longer end time horizons, it remains an open question on how to forecast for much longer timeframes. Critical applications in science and engineering would require further study both experimentally and theoretically to establish statistically reliable forecasting.

\section{Conclusions and Further work}

In this work, we addressed the challenge of long-horizon \gls{pde} forecasting via our proposed method \gls{tempo}. \Gls{tempo} consistently outperformed state-of-the-art baselines across three benchmark \gls{pde} datasets and achieves stable long-horizon 40 step forecasts with remarkable accuracy to the true trajectories as well as superior spectral fidelity. The modified time-conditioned \gls{fno} is parameter-efficient while improving the capture of both local and global spectral modes, resulting in improvements in both data- and compute- efficiency. Additionally, we establish that \gls{fno} can achieve an upper bound on approximation error that sampler-based architectures cannot reach without significantly more parameters, Corollary~\ref{cor:efficiency}.
% highlighting the theoretical expressivity established by~\cref{cor:efficiency}.
These results highlight the importance of architectures that align with the continuous nature of \gls{pde} dynamics, enabling not only improved predictive accuracy but also physically consistent, long-horizon trajectories.

Consequently, \gls{tempo} poses significant opportunity for further work in this field. Under typical real-world environments, \gls{pde} observations may come from irregularly sampled domains; since our method already demonstrates state-of-the-art generations using a simple \gls{ae} and the latent time-conditioned \gls{fno} which no longer relies on a regular grid as is a limitation of the original \gls{fno}~\citep{liFourierNeuralOperator2021}, one extension of our work is to then extend our method to real-world settings to forecast \gls{pde} over irregular domains.


\bibliography{FNOFlow}
\bibliographystyle{iclr2026_conference}

\newpage
\appendix

\section{Proofs}
\label{app:proofs}

 \begin{proof}[Proof of Theorem~\ref{thm:FNO_upper}]
\textit{Step 1: Spectral truncation.}
By assumption the Fourier coefficients of $\mathcal G(u)$ satisfy
\[
\big|\widehat{\mathcal G(u)}(k)\big|\;\le\;C_\lambda(1+|k|)^{-p}, 
\qquad \forall u\in\mathcal U,\;\;k\in\mathbb Z^d.
\]
If we keep only the modes $|k|\le K$ and set
\[
\mathcal G_K(u)(x) \;:=\; \sum_{|k|\le K}\widehat{\mathcal G(u)}(k)e^{ik\cdot x},
\]
then the error lives in the high modes:
\[
\|\mathcal G(u)-\mathcal G_K(u)\|_{H^{s'}}^2
   \;=\;\sum_{|k|>K}(1+|k|^2)^{s'}\big|\widehat{\mathcal G(u)}(k)\big|^2.
\]
Using the decay bound gives
\[
\|\mathcal G(u)-\mathcal G_K(u)\|_{H^{s'}}^2
   \;\le\;C_\lambda^2\sum_{|k|>K}(1+|k|)^{2(s'-p)}.
\]
A standard counting argument (comparing the lattice sum with a radial integral)
shows this tail is $\lesssim K^{-2\alpha}$, with
\[
\alpha \;:=\; p-s'-\tfrac d2 \;>\;0.
\]
This is exactly the pseudo-spectral tail estimate also used in \cite[Thm.~40]{kovachkiUniversalApproximationError2021}.  
Hence choosing
\[
K\;\asymp\;\varepsilon^{-1/\alpha}
\]
ensures $\|\mathcal G-\mathcal G_K\|_{H^{s'}}\le \varepsilon/2$.  

\medskip
\textit{Step 2: Reduction to a finite-dimensional map.}
The truncated operator $\mathcal G_K$ is determined by finitely many Fourier coefficients
$\{\widehat{\mathcal G(u)}(k)\}_{|k|\le K}$, with output dimension 
$m_{\rm out}\asymp K^d$. 
To apply a neural network, we also restrict the input to finitely many low modes.
By compactness of $\mathcal U\subset H^s$ and continuity of the projection $P_M$,
there exists $M$ such that
\[
\|\mathcal G_K(u)-\mathcal G_K(P_Mu)\|_{H^{s'}}\;\le\;\varepsilon/6
\qquad \forall u\in\mathcal U.
\]
This is the same finite-dimensional reduction used in the universal approximation
argument of \cite[Thm.~15]{kovachkiUniversalApproximationError2021}.  
Thus it suffices to approximate the finite-dimensional continuous map
\[
F:\;(\widehat u(k))_{|k|\le M}\;\longmapsto\;
(\widehat{\mathcal G(u)}(k))_{|k|\le K},
\]
between compact subsets of Euclidean spaces.

\medskip
\textit{Step 3: Approximation of the finite map.}
Classical universal approximation theorems (and the constructive $\Psi$--\gls{fno}
realization in \cite[Def.~11, Thm.~15]{kovachkiUniversalApproximationError2021}) ensure that for any
desired accuracy $\delta>0$, one can build a neural network (or FNO block)
approximating $F$ uniformly to error $\delta$ on each retained coefficient.  
To control the $H^{s'}$--norm it suffices to achieve coefficient accuracy
\[
\delta\;\lesssim\;\frac{\varepsilon}{K^{s'+d/2}}.
\]
This choice ensures
\(\|P_K\mathcal G(u)-\widetilde{\mathcal G}_\theta(u)\|_{H^{s'}}\le \varepsilon/3\).
Constructive approximation bounds then give a parameter count
\[
P\;\lesssim\;K^d \cdot \mathrm{polylog}(1/\varepsilon),
\]
where the extra logarithmic factor reflects standard overheads in
coefficient quantization and network approximation \cite[Remark~22]{kovachkiUniversalApproximationError2021}.

\medskip
\textit{Step 4: Assemble errors and conclude.}
Adding the contributions:
- spectral truncation error $\le\varepsilon/2$ (Step~1),
- input-projection error $\le\varepsilon/6$ (Step~2),
- finite-map approximation error $\le\varepsilon/3$ (Step~3),

we obtain
\[
\sup_{u\in\mathcal U}\|\mathcal G(u)-\mathcal G_\theta(u)\|_{H^{s'}}
   \;\le\;\varepsilon.
\]
Substituting $K\asymp\varepsilon^{-1/\alpha}$ into the parameter bound gives
\[
P_{\mathrm FNO}(\varepsilon)\;\lesssim\;\varepsilon^{-d/\alpha},
\]
up to the mild logarithmic factors discussed above.
\end{proof}

\begin{proof}[Proof of Proposition~\ref{prop:sampler_lower}]
\textit{Step 1: Finite-dimensional subspace and sampling.}  
Consider the $K$-mode Fourier subspace
\[
V_K := \operatorname{span}\{e^{ik\cdot x} : |k|\le K\} \subset L^2(\mathbb T^d),
\qquad \dim V_K =: D_K \asymp K^d.
\]
Any sampler-based learner observes an input $u\in V_K$ only through $n$ fixed points
$(u(x_1),\dots,u(x_n))$. This defines a linear map
\[
S: V_K \to  \mathbb{C}^n, \quad S(u)=(u(x_1),\dots,u(x_n)).
\]

\medskip\noindent
\textit{Step 2: Nyquist / injectivity argument.}  
To reconstruct all Fourier modes up to radius $K$, the sampling map $S$ must be
injective on $V_K$. In matrix terms, $S$ is represented by an $n\times D_K$
Vandermonde-like matrix. To have full rank $D_K$, we require
\[
n \;\ge\; D_K \;\asymp\; K^d.
\]
If $n<D_K$, there exists a nonzero $u\in V_K$ vanishing on all sample points,
so the learner cannot distinguish it from zero. This is the standard Nyquist/dimension-counting requirement: at least as many samples as degrees of freedom.

\medskip\noindent
\textit{Step 3: Parameter lower bound.}  
After sampling, the learner applies a parametric map $M: \mathbb{C}^n\to \mathbb{C}^m$ (e.g.,
a neural network) to produce either output samples or coefficients.  
To implement arbitrary linear transformations on the $D_K$ retained modes (e.g.,
arbitrary Fourier multipliers), the parametric map must have at least $D_K$ free parameters.  
For fully general dense linear maps (no structural constraints), one needs
\[
P \;\gtrsim\; D_K^2 \;\asymp\; K^{2d}.
\]

\medskip\noindent
\textit{Step 4: Conversion to accuracy $\varepsilon$.}  
From the FNO upper bound analysis, achieving accuracy $\varepsilon$ requires
\[
K \;\asymp\; \varepsilon^{-1/\alpha}, \qquad \alpha = p - s' - d/2 >0.
\]
Substituting this into the previous bounds gives the scaling
\[
n \;\gtrsim\; \varepsilon^{-\,d/\alpha}, \qquad 
P_{\rm sampler}(\varepsilon)\;\gtrsim\;\varepsilon^{-\,\beta d/\alpha},
\]
with $\beta=1$ for minimal mode-wise maps and $\beta=2$ for fully dense maps.  

\medskip\noindent
\textit{Step 5: Conclusion.}  
Hence any sampler-based architecture that must reconstruct all modes up to radius
$K$ requires asymptotically more parameters than an \gls{fno} whenever $\beta>1$, justifying
the lower bound in the proposition.
\end{proof}

\section{Flow Matching Background}\label{app:fm}
\paragraph{Flow matching} The core idea of flow matching is to learn a time-dependent velocity field, $v_\theta(z,t)$, which defines an \gls{ode} in the latent space:
\begin{equation}
\frac{dz(t)}{dt} = v_\theta(z(t),t), \quad z(0) \sim \pi_0,
\end{equation}
where $\pi_0$ is a simple reference distribution (e.g., Gaussian). Integrating this \gls{ode} transports samples to the latent data distribution $\pi_1$, such that $z(1) \sim \pi_1$ and $p_1(z) \approx f_\phi \# \mathcal{D}_{\text{data}}$, 
where $f_\phi \# \mu$ denotes the pushforward measure of a distribution $\mu$ under $f_\phi$, i.e., 
$(f_\phi \# \mu)(A) = \mu(f_\phi^{-1}(A))$ for measurable sets $A$. The corresponding time-dependent probability density, $p_t(z)$, evolves according to the continuity equation:
\begin{equation}
\frac{\partial p_t(z)}{\partial t} + \nabla_z \cdot \big( p_t(z) \, v_\theta(z,t) \big) = 0.
\end{equation}
In practice, the target velocity field $u(t,z)$ and the full marginal density $p_t(z)$ are generally unknown and intractable. Flow matching sidesteps this issue by directly supervising the model to match the instantaneous vector field along interpolating paths between the reference $\pi_0$ and the target $\pi_1$, allowing for deterministic, efficient sampling. Different choices of paths lead to different training dynamics and inductive biases, as they implicitly define the target velocity field $u(t,z)$ that the model regresses against.

Integrating this ODE from $t=0$ to $t=1$ transports the reference distribution $\pi_0$ to the latent data distribution $\pi_1$, so that $z(1) \sim \pi_1$ and $p_1(z) \approx f_\phi \# \mathcal{D}_{\text{data}}$.

\paragraph{Latent Flow Matching.} 
We now instantiate the general flow matching framework in the latent space. Let $z_\tau = f_\phi(x_\tau)$ for $\tau = 1,\dots,m$, where $f_\phi$ is a pretrained encoder mapping from the data space to the lower-dimensional latent space. Our objective is to approximate the ground-truth latent distribution $q(z_\tau \mid x_1,\dots,x_{\tau-1})$ by a parametric distribution $p(z_\tau \mid z_{\tau-1})$, which can later be decoded back to the data space via $x_\tau = g_\psi(z_\tau)$ using a decoder $g_\psi$. 

The latent dynamics can be expressed by the ODE:
\begin{equation}
\dot{z}_t = u_t(z_t),
\end{equation}
where $u_t$ denotes the (true) time-dependent velocity field. Learning these dynamics amounts to approximating $u_t$ with a neural parameterization. Following the flow matching framework, we introduce a model velocity field $v_\theta: [0,1] \times \mathbb{R}^Z \to \mathbb{R}^Z$ and consider the ODE
\begin{equation}
\dot{\phi}_t(z) = v_\theta(\phi_t(z), t), \quad \phi_0(z) = z,
\end{equation}
which defines a time-dependent diffeomorphism $\phi_t$ pushing forward an initial reference distribution $p_0$ (often chosen as $\mathcal{N}(0,I)$) to a target distribution $p_1 \approx q$ along the density path $p_t$:
\begin{equation}
p_t = (\phi_t)_\# p_0,
\end{equation}
where $(\cdot)_\#$ denotes the pushforward. In other words, the goal of flow matching is to learn a deterministic coupling between $p_0$ and $q$ by training $v_\theta$ so that the solution satisfies $z_0 \sim p_0$ and $z_1 \sim q$.

Given a probability path $p_t$ and its associated velocity field $u_t$, flow matching reduces to a least-squares regression problem:
\begin{equation}
\mathcal{L}_{\mathrm{FM}}(\theta) = \mathbb{E}_{t \sim U[0,1],\, z \sim p_t} \; \omega(t) \, \| v_\theta(z,t) - u_t(z) \|_2^2,
\end{equation}
where $\omega(t) > 0$ is a weighting function, often taken as $\omega(t) = 1$ (Lipman et al., 2022). This formulation ensures that the learned velocity field aligns with the target field $u_t$ at all times, thereby generating the desired marginal probability path.

\section{Fourier Neural Operator Background}\label{app:fno}

An \gls{fno} is designed to learn a mapping between function spaces, rather than between finite-dimensional vectors. Consider a function $u:\mathbb{R}^d \rightarrow \mathbb{R}^{c}$ representing data, for example in $\mathbb{R}^X$, with samples $x \in \mathbb{R}^X$. Then, an \gls{fno} parameterizes an operator  as
\[
\mathcal G_\theta: u \mapsto \tilde{u}, \quad \tilde{u}: \mathcal{D} \to \mathbb{R}^{c_\text{out}},
\]
that maps $u$ to an output function $\tilde{u}$ (e.g., a solution field of a \gls{pde} or a transformed spatial signal).

This mapping is implemented via iterative Fourier layers which perform spectral transformations of the input:
\begin{equation}
\hat{u}(k) = \mathcal{F}[u](k), \quad \hat{\tilde{u}}(k) = R_\theta(k) \cdot \hat{u}(k),
\end{equation}
followed by an inverse Fourier transform back to the spatial domain:
\begin{equation}
\tilde{u}(x) = \mathcal{F}^{-1}[\hat{\tilde{u}}](x),
\end{equation}
with $R_\theta(k)$ being learnable Fourier-mode weights and $\mathcal{F}$ denoting the Fourier transform. This spectral representation allows the \gls{fno} to efficiently capture long-range dependencies and global correlations in the data.

% Finally, when applied in the latent space, $\mathcal G_\theta$ can parameterize the time-dependent velocity field of flow matching:
% \begin{equation}
% \frac{dz(t)}{dt} = v_\theta(z(t),t) = \mathcal G_\theta(z(t)), \quad z(0) \sim \pi_0,
% \end{equation}
% enabling the transport of samples from a reference latent distribution $\pi_0$ to the latent data distribution $\pi_1$ while leveraging the spectral inductive bias of the \gls{fno}.
\section{Autoencoder Details}\label{app:ae}

Residual blocks throughout the architecture consist of two $3 \times 3$ convolutions with ReLU activation and group normalization (8 groups) in between, with the input added back to the output. 
Attention blocks are implemented using PyTorch’s \texttt{nn.MultiheadAttention}, with embeddings reshaped from $[B, C, H, W]$ to $[B, HW, C]$.

The autoencoder is initialised with a depth of $d=2$ resulting in a factor $2^d=4$ compression for all datasets.

\section{Model hyperparameters}\label{app:modelhypers}
We initialised the probability paths with the following hyperparameters. \gls{river} was defined with variance parameters $\sigma = 0.1$ and $\sigma_{\min} = 10^{-7}$. \gls{slp} used $\sigma = 0.1$ and $\sigma_{\min} = 0.01$. 
We further considered the \gls{ve} path with $\sigma_{\min} = 0.01$ and $\sigma_{\max} = 0.1$ and the \gls{vp} path initialized with $\beta_{\min} = 0.1$ and $\beta_{\max} = 20.0$ per~\citep{limElucidatingDesignChoice2025}.

We provide details for the vector field regressors' width and depth hyperparameters as per~\cref{tab:modelparams}.
\begin{table}[h!]
\centering
\begin{tabular}{l l l}
\toprule
\textbf{Model} & \textbf{Parameter} & \textbf{Value} \\
\midrule
\multirow{4}{*}{\gls{tempo}} 
 & $n_{\text{modes}}$ & 20 \\
 & Hidden channels & 64 \\
 & Projection channels & 64 \\
 & Depth & 4 \\
\midrule
\multirow{4}{*}{U-Net} 
 & Hidden channels & 64 \\
 & Attention resolutions & (1, 2, 2) \\
 & Channel multiplier & (1, 2, 4) \\
 & Depth & 3 \\
\midrule
\multirow{4}{*}{\gls{vit}} 
 & Hidden channels & 256 \\
 & Depth & 4 \\
 & Mid-depth & 5 \\
 & Output normalization & LayerNorm \\
\bottomrule
\end{tabular}
\caption{Descriptions of hyperparameters across \gls{tempo}, U-Net, and \gls{vit} architectures.}
\label{tab:modelparams}
\end{table}

\section{Dataset Details}
\label{app:datasets}

\begin{table}[h!]
\centering
\caption{Dataset sizes and trajectory lengths used in evaluation.}
\label{tab:datasets}
\begin{tabular}{lcc}
\toprule
Dataset & \# Trajectories & Timeseries Length \\
\midrule
\gls{swe} & 1000 & 100 \\
\gls{rd}  & 1000 & 100 \\
\gls{nsv} & 5000 & 50  \\
\bottomrule
\end{tabular}
\end{table}

\emph{\Acrfull{swe}}

The \glspl{swe} are derived from the compressible Navier--Stokes equations and model free-surface flow problems in 2D. The system of hyperbolic PDEs is given by:

\begin{align}
\partial_t h + \partial_x (h u) + \partial_y (h v) &= 0, \\
\partial_t (h u) + \partial_x \Big(u^2 h + \frac{1}{2} g_r h^2\Big) + \partial_y (u v h) &= - g_r h \, \partial_x b, \\
\partial_t (h v) + \partial_y \Big(v^2 h + \frac{1}{2} g_r h^2\Big) + \partial_x (u v h) &= - g_r h \, \partial_y b,
\end{align}

where $u, v$ are the horizontal and vertical velocities, $h$ is the water height, $b$ represents spatially varying bathymetry, and $g_r$ is gravitational acceleration. The quantities $h u$ and $h v$ correspond to directional momentum components.

The dataset (~\citep{takamotoPDEBENCHExtensiveBenchmark2022}) simulates a 2D radial dam break scenario on a square domain $\Omega = [-2.5, 2.5]^2$. The initial water height is a circular bump in the center of the domain:

\[
h(t=0, x, y) =
\begin{cases}
2.0, & \text{if } r < r_0, \\
1.0, & \text{if } r \ge r_0,
\end{cases}
\quad r = \sqrt{x^2 + y^2}, \quad r_0 \sim \mathcal{U}(0.3, 0.7).
\]

\emph{\Acrfull{rd}}

The \gls{rd} dataset models two non-linearly coupled variables: the activator $u = u(t,x,y)$ and the inhibitor $v = v(t,x,y)$. The system of \glspl{pde} is:

\begin{align}
\partial_t u &= D_u \, \partial_{xx} u + D_u \, \partial_{yy} u + R_u(u,v), \\
\partial_t v &= D_v \, \partial_{xx} v + D_v \, \partial_{yy} v + R_v(u,v),
\end{align}

where $D_u$ and $D_v$ are diffusion coefficients, and $R_u(u,v)$, $R_v(u,v)$ are the reaction functions. Specifically, the FitzHugh–Nagumo model defines the reactions as:

\begin{align}
R_u(u,v) &= u - u^3 - k - v, \\
R_v(u,v) &= u - v,
\end{align}

with $k = 5 \times 10^{-3}$, $D_u = 1 \times 10^{-3}$, and $D_v = 5 \times 10^{-3}$.  

The dataset (~\citep{takamotoPDEBENCHExtensiveBenchmark2022}) uses a simulation domain $x,y \in (-1,1)$ and $t \in (0,5]$ with initial condition set as standard normal random noise: $u(0,x,y) \sim \mathcal{N}(0,1.0)$.

\emph{\Acrfull{nsv}}

The \gls{nsv} (~\citep{liFourierNeuralOperator2021}) models 2D incompressible fluid flow on the unit torus. The system of equations is:

\begin{align}
\partial_t w(x,t) + u(x,t) \cdot \nabla w(x,t) &= \nu \, \Delta w(x,t) + f(x), \quad x \in (0,1)^2, \, t \in (0,T], \\
\nabla \cdot u(x,t) &= 0, \\
w(x,0) &= w_0(x),
\end{align}

where $w(x,t)$ is the vorticity, $u(x,t)$ is the velocity field, $\nu$ is viscosity, and $f(x)$ is a fixed forcing term:

\[
f(x) = 0.1 \Big( \sin(2 \pi (x_1 + x_2)) + \cos(2 \pi (x_1 + x_2)) \Big).
\]

The initial condition is sampled from a Gaussian measure: 

\[
w_0 \sim \mu, \quad \mu = \mathcal{N}\Big(0, \, \bigl(-\Delta + 49 I\bigr)^{-2.5} \, 7^{3/2}\Big),
\]

with periodic boundary conditions.

\section{Spectral Analysis of Ground Truth \gls{nsv}}
\label{app:spectral_ground_truth}
\begin{figure}[!htb]
    \centering
    \includegraphics[width=0.6\linewidth]{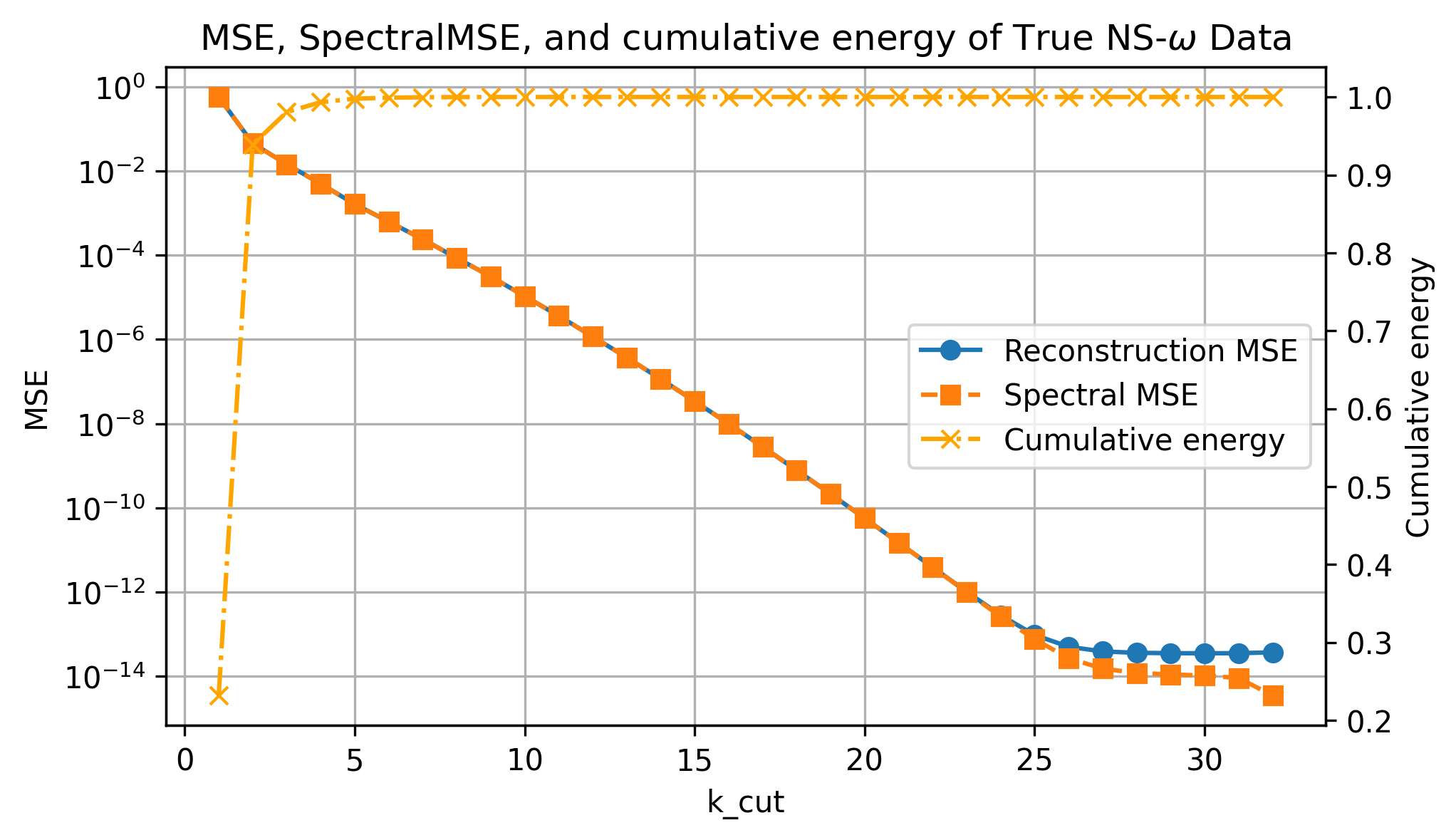}
    \caption{\textbf{Spectral Analysis of True Vorticity}: Reconstruction MSE, spectral MSE, and cumulative enstrophy fraction 
of true Navier--Stokes vorticity data as functions of cutoff wavenumber $k_{\mathrm{cut}}$. }
    \label{fig:spectral_analysis_ground_truth}
\end{figure}

\cref{fig:spectral_analysis_ground_truth} shows how the quality of spectral truncations of the true Navier--Stokes vorticity field depends on the cutoff wavenumber $k_{\mathrm{cut}}$. 
Given the full Fourier spectrum $\hat{\omega}(k_x,k_y)$, we apply a mask that retains only modes with $|k_x|+|k_y| \leq k_{\mathrm{cut}}$, reconstruct the signal by inverse FFT, and compute three quantities as functions of $k_{\mathrm{cut}}$:
\begin{enumerate}
    \item Reconstruction MSE: the mean squared error between the original and truncated fields in physical space.
    \item Spectral MSE: the mean squared error in Fourier space, quantifying lost spectral content.
    \item Cumulative energy fraction: the fraction of total energy $\sum |\hat{\omega}|^2$ retained by the truncated spectrum.
\end{enumerate}
As $k_{\mathrm{cut}}$ increases, both reconstruction and spectral errors decrease, while the retained energy approaches unity.

\newpage
\section{Extended Results for Navier–Stokes Vorticity}\label{app:nsv}

\begin{figure}[!htb]
    \centering
    \includegraphics[width=0.8\linewidth]{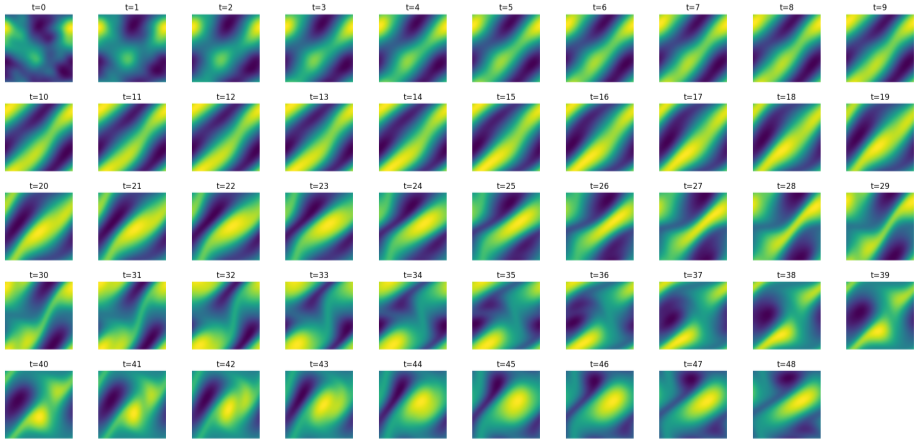}
    \caption{\textbf{Navier–Stokes Vorticity (Original).} Ground-truth timeseries across 40 timesteps.}
    \label{fig:nsv_original}
\end{figure}

\begin{figure}[!htb]
    \centering
    \includegraphics[width=0.8\linewidth]{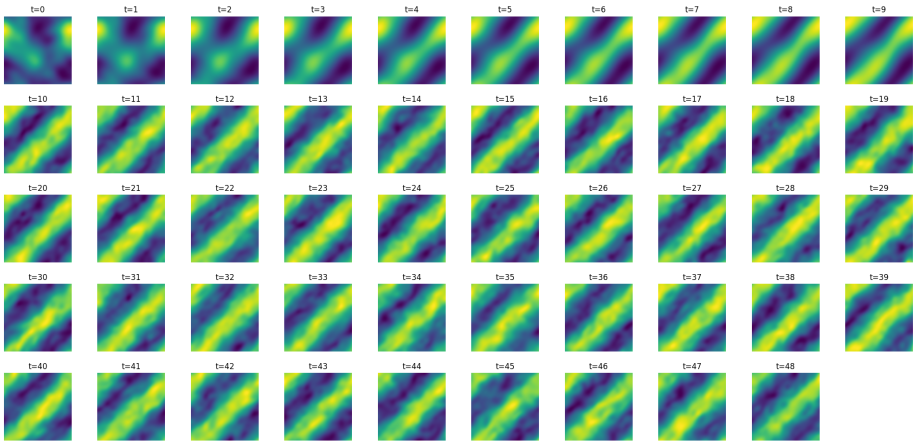}
    \caption{\textbf{Navier–Stokes Vorticity (U-Net).} Forecasted timeseries across 40 timesteps.}
    \label{fig:nsv_unet}
\end{figure}

\begin{figure}[!htb]
    \centering
    \includegraphics[width=0.8\linewidth]{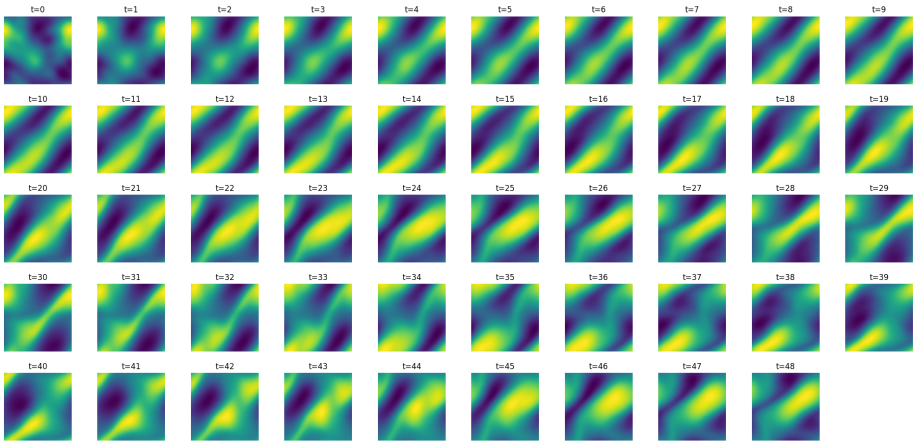}
    \caption{\textbf{Navier–Stokes Vorticity (\gls{vit}).} Forecasted timeseries across 40 timesteps.}
    \label{fig:nsv_transformer}
\end{figure}

\begin{figure}[!htb]
    \centering
    \includegraphics[width=0.8\linewidth]{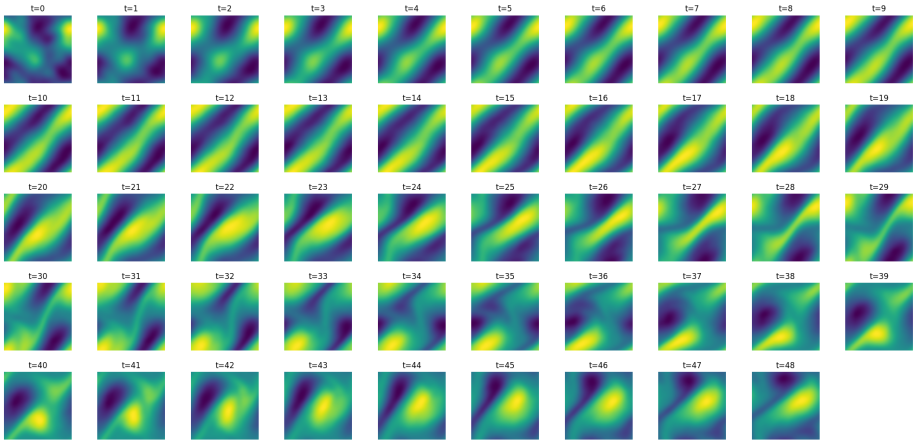}
    \caption{\textbf{Navier–Stokes Vorticity (\gls{tempo}).} Forecasted timeseries across 40 timesteps.}
    \label{fig:nsv_fno}
\end{figure}

\newpage

\section{Extended Results for Shallow Water Equation}\label{app:swe}

\begin{table}[h!]
\centering
\begin{tabular}{llllllll}
\toprule
Regressor & Path & MSE $\downarrow$ & SpectralMSE $\downarrow$ & RFNE $\downarrow$ & PSNR $\uparrow$ & Pearson $\uparrow$ & SSIM $\uparrow$ \\
\midrule
\multirow{4}{*}{\gls{tempo}} 
& \acrshort{affine} & \textbf{6.64e-05} & \textbf{5.65e-05} & \textbf{7.64e-03} & \textbf{46.5} & \textbf{0.998} & \textbf{0.997} \\
& \acrshort{river}  & 4.04e-04 & 2.33e-04 & 1.89e-02 & 38.7 & 0.989 & 0.976 \\
& \acrshort{ve}     & 9.37e-04 & 8.22e-04 & 2.89e-02 & 35.2 & 0.994 & 0.977 \\
& \acrshort{vp}     & 4.41e-03 & 2.51e-03 & 4.31e-02 & 28.3 & 0.872 & 0.857 \\
\midrule
\multirow{5}{*}{\gls{vit}} 
& \acrshort{affine} & 9.59e-05 & 7.93e-05 & 9.06e-03 & 44.9 & 0.997 & 0.995 \\
& \acrshort{vp}     & 1.30e-04 & 8.81e-05 & 1.05e-02 & 43.6 & 0.996 & 0.993 \\
& \acrshort{river}  & 2.99e-04 & 1.67e-04 & 1.63e-02 & 40.0 & 0.992 & 0.981 \\
& \acrshort{slp}\footnote{blahblah} & 6.60e-04 & - & 1.28e-01 & 36.1 & - & 0.93 \\
& \acrshort{ve}     & 1.28e-03 & 1.01e-03 & 3.38e-02 & 33.7 & 0.985 & 0.960 \\
\midrule
\multirow{3}{*}{U-Net} 
& \acrshort{vp}     & 1.37e-02 & 8.26e-03 & 1.10e-01 & 23.4 & 0.546 & 0.627 \\
& \acrshort{river}  & 1.61e-02 & 1.00e-02 & 1.20e-01 & 22.7 & 0.437 & 0.610 \\
& \acrshort{affine} & 1.68e-02 & 1.01e-02 & 1.22e-01 & 22.5 & 0.435 & 0.593 \\
\bottomrule
\end{tabular}
\caption{Comparison of \gls{tempo}, U-Net, and \gls{vit} models under different probability paths for the \gls{swe}. The best value for each metric is highlighted in bold.}
\end{table}

\begin{figure}[!htb]
    \centering
    \includegraphics[width=0.8\linewidth]{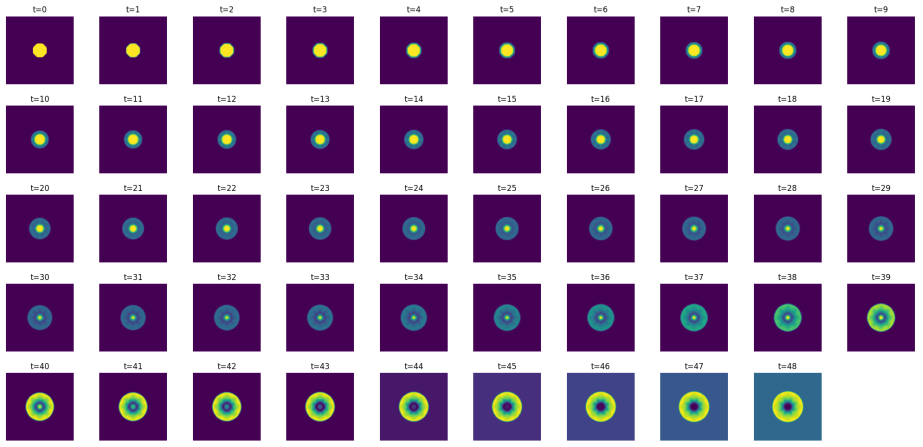}
    \caption{\textbf{SWE (Original).} Ground-truth timeseries across 40 timesteps.}
    \label{fig:swe_original}
\end{figure}

\begin{figure}[!htb]
    \centering
    \includegraphics[width=0.8\linewidth]{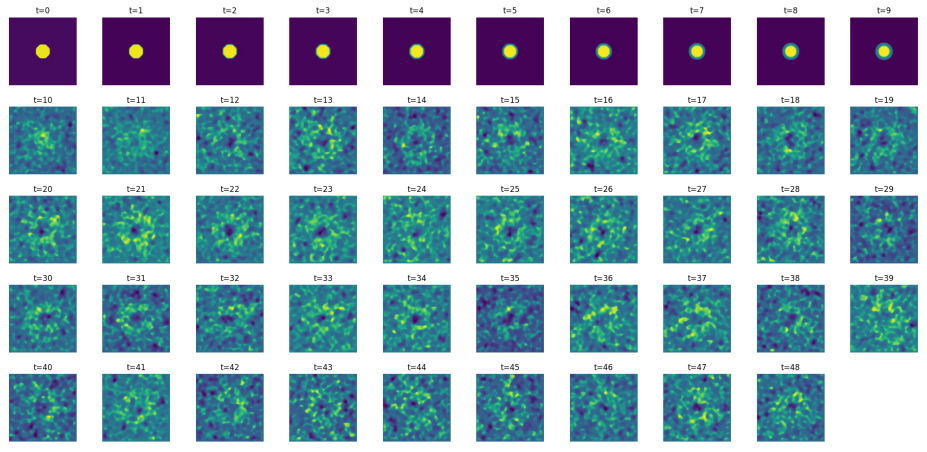}
    \caption{\textbf{SWE (U-Net).} Forecasted timeseries across 40 timesteps.}
    \label{fig:swe_unet}
\end{figure}

\begin{figure}[!htb]
    \centering
    \includegraphics[width=0.8\linewidth]{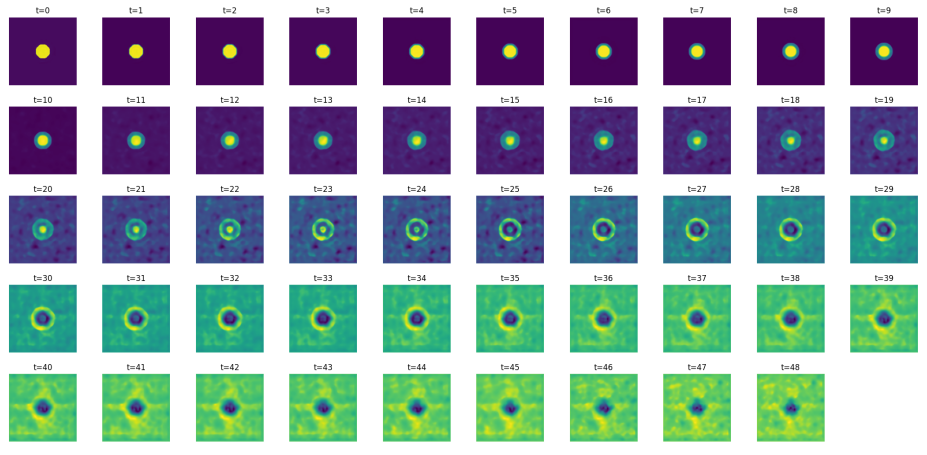}
    \caption{\textbf{SWE (\gls{vit}).} Forecasted timeseries across 40 timesteps.}
    \label{fig:swe_transformer}
\end{figure}

\begin{figure}[!htb]
    \centering
    \includegraphics[width=0.8\linewidth]{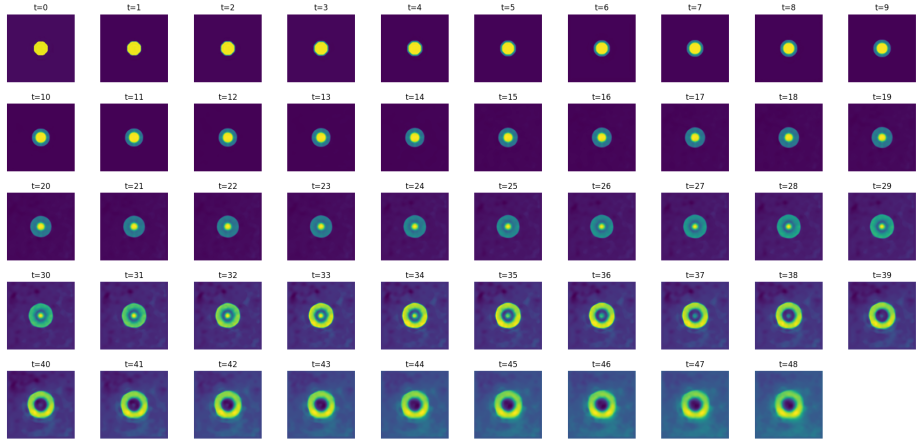}
    \caption{\textbf{SWE (\gls{tempo}).} Forecasted timeseries across 40 timesteps.}
    \label{fig:swe_fno}
\end{figure}

\newpage
\section{Extended Results for 2D Reaction Diffusion}\label{app:rd}

\begin{table}[h!]
\centering
\begin{tabular}{llllllll}
\toprule
Regressor & Path & MSE $\downarrow$ & SpectralMSE $\downarrow$ & RFNE $\downarrow$ & PSNR $\uparrow$ & Pearson $\uparrow$ & SSIM $\uparrow$ \\
\midrule
\multirow{4}{*}{\gls{tempo}} 
& \acrshort{affine} & \textbf{2.76e-05} & \textbf{2.18e-05} & \textbf{3.29e-02} & \textbf{65.7} & \textbf{1.000} & \textbf{0.999} \\
& \acrshort{river}  & 9.36e-04 & 5.47e-04 & 2.08e-01 & 50.4 & 0.975 & 0.978 \\
& \acrshort{ve}     & 1.58e-03 & 1.38e-03 & 2.70e-01 & 48.2 & 0.990 & 0.977 \\
& \acrshort{vp}     & 1.24e-02 & 1.01e-02 & 4.95e-01 & 39.2 & 0.714 & 0.862 \\
\midrule
\multirow{4}{*}{\gls{vit}} 
& \acrshort{affine} & 6.30e-04 & 4.40e-04 & 1.67e-01 & 52.2 & 0.987 & 0.986 \\
& \acrshort{slp}\footnote{Results taken from experiments in~\citep{limElucidatingDesignChoice2025} over same dataset.} & 3.56e-04 & - & 1.16e-01 & 34.3 & - & 0.90 \\
& \acrshort{river}  & 1.00e-03 & 5.89e-04 & 2.16e-01 & 50.1 & 0.973 & 0.977 \\
& \acrshort{ve}     & 3.54e-03 & 2.23e-03 & 4.06e-01 & 44.7 & 0.915 & 0.946 \\
\midrule
\multirow{4}{*}{U-Net} 
& \acrshort{affine} & 3.09e-05 & 2.45e-05 & 3.57e-02 & 65.2 & 0.999 & \textbf{0.999} \\
& \acrshort{river}  & 1.02e-03 & 5.49e-04 & 2.17e-01 & 50.1 & 0.972 & 0.976 \\
& \acrshort{ve}     & 9.03e-03 & 6.07e-03 & 6.42e-01 & 40.6 & 0.820 & 0.860 \\
& \acrshort{vp}     & 2.09e-02 & 1.66e-02 & 6.81e-01 & 37.0 & 0.574 & 0.792 \\
\bottomrule
\end{tabular}
\caption{Comparison of \gls{tempo}, U-Net, and \gls{vit} models under different probability paths for the \gls{rd}. The best value for each metric is highlighted in bold.}
\end{table}

\begin{figure}[!htb]
    \centering
    \includegraphics[width=0.3\linewidth]{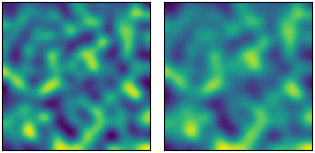}
    \caption{\textbf{Reaction Diffusion (Original).} Ground-truth end sample, from initial conditions of randomly sampled noise.}
    \label{fig:reacdiff_original}
\end{figure}

\begin{figure}[!htb]
    \centering
    \includegraphics[width=0.3\linewidth]{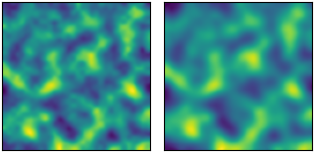}
    \caption{\textbf{Reaction Diffusion (U-Net).} Forecasted end sample, from initial conditions of randomly sampled noise.}
    \label{fig:reacdiff_unet}
\end{figure}

\begin{figure}[!htb]
    \centering
    \includegraphics[width=0.3\linewidth]{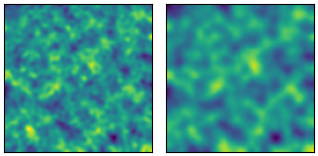}
    \caption{\textbf{Reaction Diffusion (\gls{vit}).} Forecasted end sample, from initial conditions of randomly sampled noise.}
    \label{fig:reacdiff_transformer}
\end{figure}

\begin{figure}[!htb]
    \centering
    \includegraphics[width=0.3\linewidth]{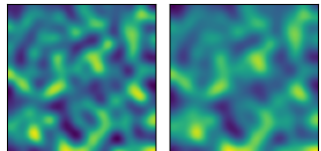}
    \caption{\textbf{Reaction Diffusion (\gls{tempo}).} Forecasted end sample, from initial conditions of randomly sampled noise.}
    \label{fig:reacdiff_fno}
\end{figure}

\clearpage

\section{Extended Ablation Results}\label{app:ablation}
\begin{table}[h!]
\centering
\caption{Ablation over different training sequence lengths on the \gls{nsv} dataset. 
\gls{tempo} and the top performing alternative are trained while varying sequence lengths and evaluated on $10$ timesteps to predict the next step.}
\label{tab:seq_length_ablation}
\begin{tabular}{llrrrrrrrrr}
\toprule
Method & Seq.~Len. & MSE & DensityMSE & SpectralMSE & RFNE & PSNR & Pearson & SSIM & NFE \\
\midrule
\multirow{5}{*}{\gls{tempo}} 
 & 3  & 4.924e-02 & 7.685e-05 & 3.531e-02 & 2.328e-01 & 25.769 & 0.973 & 0.803 & 74 \\
 & 6  & 4.753e-02 & 1.133e-04 & 3.394e-02 & 2.276e-01 & 25.923 & 0.974 & 0.800 & 608 \\
 & 11 & 5.036e-02 & 1.055e-04 & 3.620e-02 & 2.352e-01 & 25.672 & 0.972 & 0.800 & 842 \\
 & 16 & 5.607e-02 & 1.282e-04 & 3.821e-02 & 2.497e-01 & 25.205 & 0.969 & 0.786 & 938 \\
 & 26 & 6.255e-02 & 7.487e-05 & 3.726e-02 & 2.541e-01 & 24.730 & 0.968 & 0.765 & 1070 \\
\midrule
\multirow{5}{*}{\makecell{\gls{vit} \\ (\gls{affine})}}
 & 3  & 6.748e-02 & 1.414e-04 & 4.652e-02 & 2.678e-01 & 24.401 & 0.963 & 0.766 & 116 \\
 & 6  & 5.434e-02 & 1.239e-04 & 3.727e-02 & 2.416e-01 & 25.341 & 0.970 & 0.783 & 1766 \\
 & 11 & 6.014e-02 & 1.376e-04 & 4.067e-02 & 2.546e-01 & 24.901 & 0.967 & 0.777 & 1712 \\
 & 16 & 6.701e-02 & 1.093e-04 & 4.428e-02 & 2.680e-01 & 24.431 & 0.963 & 0.764 & 1622 \\
 & 26 & 7.682e-02 & 8.104e-05 & 4.468e-02 & 2.778e-01 & 23.838 & 0.960 & 0.741 & 1100 \\
\bottomrule
\end{tabular}
\end{table}

\begin{table}[h!]
\centering
\caption{Ablation of the \gls{tempo} model on the \gls{nsv} dataset by varying the number of modes. 
Models are trained with different numbers of Fourier modes and evaluated on 10 timesteps to predict the next step.}
\label{tab:tempo_modes_ablation}
\begin{tabular}{rllllllll}
\toprule
Modes & MSE & DensityMSE & SpectralMSE & RFNE & PSNR & Pearson & SSIM & NFE \\
\midrule
1  & 1.409e-01 & 1.075e-04 & 8.566e-02 & 3.947e-01 & 21.204 & 0.921 & 0.588 & 5798 \\
2  & 6.103e-02 & 8.928e-05 & 4.096e-02 & 2.596e-01 & 24.837 & 0.966 & 0.765 & 1688 \\
4  & 5.789e-02 & 8.361e-05 & 3.978e-02 & 2.538e-01 & 25.066 & 0.968 & 0.776 & 1058 \\
8  & 5.528e-02 & 8.498e-05 & 3.788e-02 & 2.481e-01 & 25.267 & 0.969 & 0.788 & 800 \\
16 & 5.471e-02 & 8.757e-05 & 3.742e-02 & 2.467e-01 & 25.312 & 0.970 & 0.787 & 884 \\
\bottomrule
\end{tabular}
\end{table}

\section{Use of Large Language Models (LLMs)}
We acknowledge the use of ChatGPT to make suggestions on how to polish the text, correct grammar, and ensure clarity in writing. No results, code, or data were created or altered by the model.

\end{document}